\documentclass[11pt,letterpaper]{article}
\usepackage[T1]{fontenc}
\usepackage[utf8]{inputenc}

\usepackage{amsmath,amssymb,amsthm,mathtools}
\usepackage{newpxtext}                 % Palatino text
\usepackage{newpxmath}                 % Palatino math
\usepackage[scaled=0.90]{FiraSans}     % headings, labels, title block

\usepackage[letterspace=80]{microtype}

\usepackage[letterpaper,left=1.15in,right=1.15in,top=1.1in,bottom=1.15in,headsep=0.32in,footskip=0.45in]{geometry}
\usepackage{xcolor}
\definecolor{ink}{HTML}{1A1F2B}
\definecolor{navy}{HTML}{1F3A63}         % headings, labels
\definecolor{refcool}{HTML}{2C5AA0}      % cross-references, numbers, theorem bars
\definecolor{citewarm}{HTML}{B5403B}     % citations, subtitle, definition bars
\definecolor{muted}{HTML}{6B6A66}        % running heads, meta text
\definecolor{rulegray}{HTML}{D3D6DC}
\definecolor{thmback}{HTML}{F2F6FC}
\definecolor{lemback}{HTML}{F5F6F8}
\definecolor{lembar}{HTML}{6B7A90}
\definecolor{absback}{HTML}{F7F8FA}

\usepackage{graphicx,booktabs,array,tabularx,multicol}
\usepackage{enumitem}
\setlist{itemsep=1pt,parsep=0pt,topsep=3pt,partopsep=0pt}
\usepackage[font=small,labelfont={sf,bf,color=navy},labelsep=period,skip=6pt]{caption}
\usepackage[round]{natbib}

\usepackage[most]{tcolorbox}
\usepackage{titlesec,titletoc}
\usepackage{fancyhdr}
\usepackage{algorithm,algorithmic}
\usepackage{url}
\usepackage{aliascnt}

\titleformat{\section}{\sffamily\Large\bfseries\color{navy}}{\textcolor{refcool}{\thesection}}{0.7em}{}
\titlespacing*{\section}{0pt}{3.2ex plus 1ex minus .2ex}{1.5ex plus .2ex}
\titleformat{\subsection}{\sffamily\large\bfseries\color{navy}}{\textcolor{refcool}{\thesubsection}}{0.6em}{}
\titlespacing*{\subsection}{0pt}{2.4ex plus .8ex minus .2ex}{1.0ex plus .2ex}
\titleformat{\subsubsection}{\sffamily\normalsize\bfseries\color{navy}}{\textcolor{refcool}{\thesubsubsection}}{0.5em}{}
\titleformat{\paragraph}[runin]{\sffamily\bfseries\color{navy}}{}{0pt}{}
\titlespacing*{\paragraph}{0pt}{1.4ex plus .5ex minus .2ex}{0.55em}

\titlecontents{section}[2.0em]{\addvspace{0.6pt}\sffamily\fontsize{8.4}{10.2}\selectfont}
  {\contentslabel[\textcolor{refcool}{\bfseries\thecontentslabel}]{2.0em}}{\hspace*{-2.0em}}
  {\titlerule*[0.6pc]{\textcolor{rulegray}{.}}\contentspage}

\fancypagestyle{first}{\fancyhf{}\fancyfoot[C]{\sffamily\footnotesize\color{muted}\thepage}}

\newtheoremstyle{boxed}{0pt}{0pt}{\itshape}{}{\sffamily\bfseries\color{navy}}{.}{0.5em}
  {\thmname{#1}\thmnumber{ #2}\thmnote{\ \textmd{(#3)}}}
\newtheoremstyle{boxedrm}{0pt}{0pt}{\normalfont}{}{\sffamily\bfseries\color{navy}}{.}{0.5em}
  {\thmname{#1}\thmnumber{ #2}\thmnote{\ \textmd{(#3)}}}
\theoremstyle{boxed}
\newtheorem{theorem}{Theorem}[section]
\newtheorem{proposition}[theorem]{Proposition}
\newtheorem{lemma}[theorem]{Lemma}
\newtheorem{corollary}[theorem]{Corollary}

\theoremstyle{boxedrm}
\newtheorem{definition}[theorem]{Definition}
\newtheorem{assumption}[theorem]{Assumption}
\newtheorem{remark}[theorem]{Remark}
\newtheorem{example}[theorem]{Example}

\tcbset{thmbox/.style={enhanced,breakable,frame hidden,sharp corners,boxrule=0pt,
  left=8pt,right=8pt,top=5pt,bottom=5pt,before skip=9pt,after skip=9pt}}
\tcbset{coolbox/.style={thmbox,colback=thmback,borderline west={2.4pt}{0pt}{refcool}}}
\tcbset{slatebox/.style={thmbox,colback=lemback,borderline west={2.4pt}{0pt}{lembar}}}
\tcbset{warmbox/.style={thmbox,colback=white,borderline west={2.4pt}{0pt}{citewarm}}}
\tcbset{linebox/.style={blanker,breakable,left=8pt,before skip=7pt,after skip=7pt,borderline west={1pt}{0pt}{rulegray}}}
\tcolorboxenvironment{theorem}{coolbox}
\tcolorboxenvironment{proposition}{coolbox}
\tcolorboxenvironment{corollary}{coolbox}
\tcolorboxenvironment{conjecture}{coolbox}
\tcolorboxenvironment{lemma}{slatebox}
\tcolorboxenvironment{claim}{slatebox}
\tcolorboxenvironment{fact}{slatebox}
\tcolorboxenvironment{observation}{slatebox}
\tcolorboxenvironment{definition}{warmbox}
\tcolorboxenvironment{assumption}{warmbox}
\tcolorboxenvironment{remark}{linebox}
\tcolorboxenvironment{example}{linebox}
\tcolorboxenvironment{problem}{linebox}
\tcolorboxenvironment{question}{linebox}

\makeatletter
\renewenvironment{proof}[1][\proofname]{\par\pushQED{\qed}\normalfont\topsep6\p@\@plus6\p@\relax
  \trivlist\item[\hskip\labelsep{\sffamily\bfseries\color{navy}#1.}]\ignorespaces}{\popQED\endtrivlist\@endpefalse}
\makeatother

\providecommand{\markmainend}{}

\makeatletter
\newcommand{\N}{\mathbb{N}}
\newcommand{\R}{\mathbb{R}}
\newcommand{\eps}{\varepsilon}
\newcommand{\wh}{\widehat}

\newcommand{\one}{\mathbf{1}}
\newcommand{\cA}{\mathcal{A}}

\newcommand{\cC}{\mathcal{C}}
\newcommand{\cD}{\mathcal{D}}
\newcommand{\cE}{\mathcal{E}}

\newcommand{\cG}{\mathcal{G}}
\newcommand{\cN}{\mathcal{N}}
\newcommand{\cR}{\mathcal{R}}
\newcommand{\cU}{\mathcal{U}}
\newcommand{\err}{\mathrm{err}}
\newcommand{\EM}{\mathrm{EM}}
\newcommand{\GN}{\mathrm{GN}}
\newcommand{\Unif}{\mathrm{Unif}}
\newcommand{\dH}{d_{\mathrm{H}}}
\newcommand{\cgap}{c_{\mathrm{gap}}}
\newcommand{\qid}{q^{\mathrm{id}}}
\newcommand{\qgen}{q^{\lambda}}
\DeclareMathOperator*{\E}{\mathbb{E}}
\DeclareMathOperator{\supp}{supp}

\DeclareMathOperator{\iderr}{IdErr}
\DeclareMathOperator{\generr}{GenErr}
\makeatother

\usepackage[colorlinks=true,linkcolor=refcool,citecolor=citewarm,urlcolor=refcool,
  pdftitle={On the Price of Privacy for Language Identification and Generation},pdfauthor={Xiaoyu Li, Andi Han, Jiaojiao Jiang, Junbin Gao},pdfsubject={},pdfkeywords={differential privacy, language identification, language generation, universal rates, learning theory, exponential mechanism, Le Cam's two-point method}]{hyperref}

\begin{document}
\thispagestyle{first}

\noindent{\color{refcool}\rule{\linewidth}{2.2pt}}\par\vspace{7pt}
\noindent{\sffamily\footnotesize\color{muted}\lsstyle NeurIPS 2026 $\cdot$ Extended Version\quad\textbullet\quad November 2026}\par\vspace{8pt}
\noindent{\sffamily\bfseries\fontsize{23}{28}\selectfont\color{navy}On the Price of Privacy for Language\\ Identification and Generation\par}\vspace{7pt}
\vspace{6pt}
\noindent\begin{tabular}{@{}>{\raggedright\arraybackslash}p{\dimexpr(\linewidth-10pt)/2\relax}@{\hspace{10pt}}>{\raggedright\arraybackslash}p{\dimexpr(\linewidth-10pt)/2\relax}@{}}
{\sffamily\bfseries\color{ink}Xiaoyu Li} & {\sffamily\bfseries\color{ink}Andi Han} \\[1pt]
{\small University of New South Wales} & {\small University of Sydney} \\[1pt]
{\sffamily\footnotesize\href{mailto:xiaoyu.li2@unsw.edu.au}{\color{muted}xiaoyu.li2@unsw.edu.au}} & {\sffamily\footnotesize\href{mailto:andi.han@sydney.edu.au}{\color{muted}andi.han@sydney.edu.au}}
\end{tabular}\par\vspace{6pt}\noindent\begin{tabular}{@{}>{\raggedright\arraybackslash}p{\dimexpr(\linewidth-10pt)/2\relax}@{\hspace{10pt}}>{\raggedright\arraybackslash}p{\dimexpr(\linewidth-10pt)/2\relax}@{}}
{\sffamily\bfseries\color{ink}Jiaojiao Jiang} & {\sffamily\bfseries\color{ink}Junbin Gao} \\[1pt]
{\small University of New South Wales} & {\small University of Sydney} \\[1pt]
{\sffamily\footnotesize\href{mailto:jiaojiao.jiang@unsw.edu.au}{\color{muted}jiaojiao.jiang@unsw.edu.au}} & {\sffamily\footnotesize\href{mailto:junbin.gao@sydney.edu.au}{\color{muted}junbin.gao@sydney.edu.au}}
\end{tabular}

\begingroup\renewcommand\thefootnote{}\footnotetext{A preliminary version of this paper was accepted at the 40th Conference on Neural Information Processing Systems (NeurIPS 2026).}\endgroup
\vspace{10pt}

\begin{tcolorbox}[enhanced,colback=absback,frame hidden,sharp corners,boxrule=0pt,
  borderline west={3pt}{0pt}{refcool},left=11pt,right=11pt,top=7pt,bottom=7pt]
{\sffamily\bfseries\footnotesize\color{navy}\lsstyle ABSTRACT}\par\vspace{3pt}
\small Large language models train on sensitive user data, so we ask what differential privacy (DP) costs in language learning.
We study private language identification and generation in the agnostic statistical setting.
For identification, privacy is free: at every fixed $\eps$, pure DP and $(\eps,\delta)$-DP with $\delta\ge n^{-O(1)}$ keep the non-private rate $\exp(-r(n))$ for every $r(n)=o(n)$.
For generation, the price depends on what is public.
With a public witness bound, pure DP loses at most a factor $\Theta(\min\{1,\eps\})$ in the exponent, and on suitable two-language collections the optimal rate is $\exp(-\Theta(\min\{1,\eps\}\,n))$.
Without such a bound, pure DP can lose the exponential rate itself, already for two languages: every rate $\exp(-o(n))$ stays achievable, but no exponential rate is.
On these collections, $(\eps,\delta)$-DP keeps an exponential rate if and only if $\log(1/\delta)=o(n)$.
Three facts about pure DP drive the proofs. It turns a robustness margin into an exponentially small error. It obeys a private form of Le~Cam's two-point bound. And it cannot see the absence of a string.\par
\vspace{4pt}{\sffamily\footnotesize\color{navy}\textbf{Keywords:}}\ {\footnotesize differential privacy, language identification, language generation, universal rates, learning theory, exponential mechanism, Le Cam's two-point method}
\end{tcolorbox}

\vspace{2pt}
\begin{tcolorbox}[enhanced,colback=white,colframe=rulegray,boxrule=0.5pt,arc=2pt,left=10pt,right=10pt,top=5pt,bottom=4pt]
{\sffamily\bfseries\footnotesize\color{navy}\lsstyle CONTENTS}\par
\setlength{\multicolsep}{4pt plus 1pt minus 1pt}\begin{multicols}{2}
{\hypersetup{linkcolor=ink}\makeatletter\@starttoc{toc}\makeatother}
\end{multicols}
\end{tcolorbox}
\suppressfloats[t]
\clearpage

\suppressfloats[t]
\section{Introduction}\label{sec:intro}

Differential privacy~\citep{dmns06} is the standard framework for learning from sensitive data.
Its statistical cost depends on the task: some problems keep their non-private rates under privacy, and others lose a constant factor or more.
Language learning is now an urgent case, because large language models train on user text.
In practice, private fine-tuning comes close to non-private baselines~\citep{ltlh2022large,ynb+22}, and \citet{rya+25} measure the compute--privacy--utility trade-off of private training.
The theoretical cost of privacy for language learning is still open.
This paper quantifies it, with upper and lower bounds, for the two basic tasks, identification and generation.

\citet{gold67} introduced identification in the limit, and \citet{angluin79,angluin80} characterized it.
Their results show strong barriers: for example, the class of regular languages is not identifiable in the limit from positive examples.
\citet{km24} changed the goal to \emph{generation}: the learner only has to produce a new valid string.
This is possible in the limit for every countable collection.
These results use an online model with adversarial examples.
\citet{kmv25} moved the problem to a statistical model with i.i.d.\ samples, and \citet{hp26} removed the realizability assumption.
We work in this agnostic statistical model.

A learner receives $n$ i.i.d.\ samples $S=(x_1,\dots,x_n)\sim\cD^n$ from an unknown distribution $\cD$ on a countable universe $\cU$.
It also knows a countable collection $\cC=\{L_1,L_2,\dots\}$ of languages $L_i\subseteq\cU$.
There are two tasks:
\begin{itemize}
\item \textbf{Identification.} Output a language $\wh L\in\cC$ whose risk $\Pr_{x\sim\cD}[x\notin\wh L]$ is close to the smallest risk in $\cC$.
\item \textbf{Generation.} Output a string $\wh x\in\cU$ that is \emph{valid} ($\wh x\in\supp(\cD)$) and \emph{new} ($\wh x\notin S$).
\end{itemize}
Without privacy, \citet{hp26} prove two rates.
If some language attains the best risk in $\cC$, the identification error decays as $\exp(-r(n))$ for every $r(n)=o(n)$.
Under two structural assumptions, the generation error decays as $\exp(-\Theta(n))$.
We ask:
\begin{center}
\emph{What does differential privacy cost for language identification and generation?}
\end{center}

\subsection{Main results}\label{sec:results}

Table~\ref{tab:summary} summarizes the answer.
Privacy is free except for generation under pure DP.
There, with a public witness bound, it costs at most a factor $\Theta(\min\{1,\eps\})$ in the exponent, and on suitable two-language collections this factor is necessary (Example~\ref{ex:sharp}); without a public bound, it costs, for some collections, the exponential rate itself.
On these collections, approximate DP keeps the exponential rate exactly when $\log(1/\delta)=o(n)$.

\begin{table}[t]
\centering
\caption{Error rates for a fixed distribution, with constants that depend on $\cD$ and $\cC$, as for universal rates~\citep{bhmvy21}.
An upper bound $e^{-o(n)}$ means that for every $r(n)=o(n)$ some algorithm, which may depend on $r$, has error $O(e^{-r(n)})$~\citep{hm26}.
$^\dagger$\citet{hp26}.
$^\star$Error at most $4e^{-r(n)}$ for $r(n)\ge\log n$ if $\eps\ge2\sqrt{r(n)/n}$ (pure DP), or if $\eps<1$ and $\eps^2\sqrt{n\,r(n)}\ge2\log(1.25/\delta)$ (approximate DP), so $\eps$ may decrease with $n$ (Corollary~\ref{cor:id-rates}).
$^\circ$For every fixed $\eps>0$ and $\delta\ge n^{-O(1)}$; Theorem~\ref{thm:gen-public} and Corollary~\ref{cor:gen-rates} give the exact conditions.
$^\diamond$With a public witness bound, or a known mass floor (Section~\ref{sec:gen}).
$^\sharp$For finite collections with two infinite languages of finite intersection, exactly $e^{-o(n)}$~\citep{hm26}: every consistent pure-DP generator has error at least $\frac12e^{-cn}$ for every $c>0$ and large $n$ (Corollary~\ref{cor:gen-threshold}, Theorem~\ref{thm:gen-nec}).
$^\flat$On the same collections, if and only if $\log(1/\delta)=o(n)$.
$^\ddagger$Our lower bounds hold for all $\eps\in(0,\infty]$ and every $n$, on collections with an intersecting pair (Section~\ref{sec:lb}); at $\eps=\infty$ they give the non-private ones, to which every approximate-DP algorithm is subject.}
\label{tab:summary}
\small
\renewcommand{\arraystretch}{1.15}
\begin{tabular}{@{}llll@{}}
\toprule
 & \textbf{Non-private}$^\dagger$ & \textbf{Pure $\eps$-DP} & \textbf{Approximate $(\eps,\delta)$-DP} \\
\midrule
Identification, upper & $e^{-o(n)}$ & $e^{-o(n)}$ $^\star$ & $e^{-o(n)}$ $^\star$ \\
Identification, lower$^\ddagger$ & $e^{-O(n)}$ & $e^{-O(\min\{1,\eps\}n)}$ & $e^{-O(n)}$ \\
\midrule
Generation, upper & $e^{-\Omega(n)}$ & $e^{-\Omega(\min\{1,\eps\}n)}$ $^\diamond$ & $e^{-\Omega(n)}$ $^\circ$ \\
Generation, lower$^\ddagger$ & $e^{-O(n)}$ & $e^{-O(\min\{1,\eps\}n)}$ & $e^{-O(n)}$ \\
Generation, no public bound & $e^{-\Theta(n)}$ & $e^{-o(n)}$ $^\sharp$ & $e^{-\Theta(n)}$ $^\flat$ \\
\bottomrule
\end{tabular}
\end{table}

Our contributions are as follows.
\begin{itemize}
\item \textbf{Three master lemmas (Section~\ref{sec:setup}).}
A \emph{private selection lemma} bounds the error of two standard mechanisms by the \emph{margin} of the best candidate, in units of the sensitivity.
A \emph{private two-point inequality} states that every $\eps$-DP algorithm errs with probability at least $\E[e^{-\eps\,\dH(S,S')}]$ on a coupled pair of samples.
At $\eps=\infty$, this inequality is Le~Cam's two-point bound.
A \emph{private resampling inequality} shows that pure DP cannot see the absence of a string.
\item \textbf{Identification (Section~\ref{sec:id}).}
We replace the margin test of \citet{hp26} by its \emph{signed margin}, a score with sensitivity $2/n$ and margin $1/f(n)$.
Pure and approximate DP then keep the non-private rate at every fixed $\eps$.
\item \textbf{Generation (Section~\ref{sec:gen}).}
One family of \emph{capped-depth scores}, all with sensitivity $1$, covers the cases with and without a public witness bound.
With a public bound, pure DP achieves $\exp(-\Omega(\min\{1,\eps\}n))$ and needs no knowledge of the string masses.
Without a public bound, there are collections, already with two languages, on which no pure-DP generator is exponential, while approximate DP is exponential exactly when $\log(1/\delta)=o(n)$.
So, for such collections, a public witness bound is worth an exponential under pure DP and nothing under approximate DP with $\log(1/\delta)=o(n)$.
\item \textbf{Lower bounds (Section~\ref{sec:lb}).}
For both tasks, on collections with an intersecting pair of languages, every $\eps$-DP algorithm has error at least $c\,(1-\tfrac14(1-e^{-\eps}))^n$ on one of two distributions, for every $n$.
Because $1-e^{-\eps}\asymp\min\{1,\eps\}$, the factor $\min\{1,\eps\}$ is necessary at the linear scale; for generation with a public witness bound, the upper and lower bounds match up to constants in the exponent on suitable two-language collections (Example~\ref{ex:sharp}).
\end{itemize}

\paragraph{What is inherited, what is standard, and what is new.}
The setting, the non-private rules, and their rates are from \citet{hp26}.
The exponential mechanism, the Gaussian mechanism, and group privacy are standard~\citep{mt07,dkm+06,dr14}.
Our new ingredients are the following:
\begin{itemize}
\item two scores that keep the decision logic of the non-private rules at sensitivity $2/n$ and $1$. The signed margin avoids the index reward of the natural soft-margin score, and this makes pure DP free for identification (Remark~\ref{rem:index-reward});
\item a sharp private resampling inequality, and a private two-point inequality that refines the coupling method of \citet{asz21};
\item the analysis of generation without a public bound, which ends in a separation between pure and approximate DP (Theorem~\ref{thm:gen-nec}).
\end{itemize}

\paragraph{Why the two tasks pay differently.}
Each algorithm in this paper selects a candidate by a score, and two numbers govern its error.
The first number is the large-deviation exponent of the sampling event on which the right candidate wins; this exponent gives the non-private rate.
The second number is the margin $\rho$ of the win, in samples: an adversary must change about $\rho$ entries of $S$ to flip the decision.
Pure DP fails with probability about $e^{-\eps\rho}$ (Lemma~\ref{lem:selection}).
Conversely, let two instances with different answers differ in about $\rho$ samples.
Then every $\eps$-DP algorithm fails with probability about $e^{-\eps\rho}$ on one of them (Lemma~\ref{lem:lecam}).

In identification, the decision compares empirical risks at precision $\eta=1/f(n)$.
Sampling noise resolves this precision at exponent $n\eta^2$, but the adversary needs $n\eta$ changes, so privacy costs nothing once $\eps\ge\eta$.
In generation, the decision compares the counts of single strings.
A string of mass $p$ is absent from $S$ with probability about $e^{-np}$, and an adversary needs about $np$ changes to hide it.
Both exponents scale as $np$, so pure DP multiplies the exponent by $\Theta(\min\{1,\eps\})$.
Gaussian noise has quadratic tails, so approximate DP pays only $e^{-\Theta(\eps^2\rho^2/(|\cR|\log(1/\delta)))}$ over $|\cR|$ candidates, which is negligible in both tasks when the number of candidates grows slowly.
Figure~\ref{fig:price} shows the two regimes.
A third effect appears when no witness bound is public. The learner must then decide which strings are absent, and pure DP cannot do this at an exponential rate (Lemma~\ref{lem:resampling} and Theorem~\ref{thm:gen-nec}).

\begin{figure}[t]
\centering
\includegraphics[width=\linewidth]{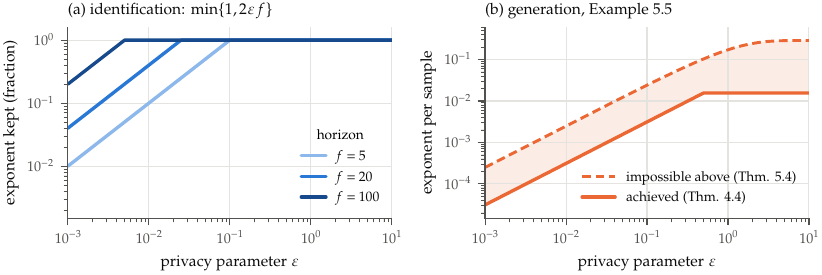}
\caption{The price of pure privacy, read off our theorems.
(a) Identification: the fraction of the non-private exponent that Theorem~\ref{thm:id}(a) keeps is $\min\{1,2\eps f\}$, because the privacy exponent $\eps n/(4f)$ overtakes the sampling exponent $n/(8f^2)$ at $\eps=1/(2f)$.
The horizon $f(n)$ grows with $n$, so privacy is free at every fixed $\eps$.
(b) Generation on the pair of Example~\ref{ex:sharp}: Theorem~\ref{thm:gen-public}(a) achieves the exponent $\min\{1/64,\eps/32\}$ per sample, and Theorem~\ref{thm:gen-lb} rules out any exponent above $-\log(1-\theta_\eps/4)$.
Both edges of the band grow linearly in $\eps$ and then level off, so on this pair the factor $\Theta(\min\{1,\eps\})$ is sharp.}
\label{fig:price}
\end{figure}

\subsection{Related work}\label{sec:related}

\paragraph{Language identification and generation.}
\citet{gold67} started the formal study of language learnability, and \citet{angluin79,angluin80} characterized it.
\citet{km24} introduced generation in the limit.
Follow-up work studies the trade-off between diversity and hallucination, robustness to noise, and computational barriers~\citep{kmv25,kmv26,cp24,kw25a,kw25,rr25,bpz25,mvyz25,abck25}.
\citet{kmv25} and \citet{hp26} study the statistical model, the first under realizability and the second in the agnostic case that we adopt.
Concurrent work by \citet{mvyz26} studies pure-DP generation and identification in the online model, in the limit, for realizable sequences and under continual release.
There, every countable collection stays privately generatable.
Adversarial identification fails as soon as two languages have an infinite intersection and a finite difference, and Angluin's condition characterizes stochastic identification.
We work in the agnostic statistical model instead and determine rates under both pure and approximate DP, so the two studies complement each other.
Appendix~\ref{app:related} reviews more work.

\paragraph{Private selection and robustness.}
Our identification algorithms are private model selection over a countable class with a growing horizon $f(n)$.
\citet{bksw19} and \citet{gkk+20} study private hypothesis selection over finite classes, in the central and the local model.
Both our upper bounds turn a robustness margin into a privacy guarantee.
The same principle drives the inverse sensitivity mechanism~\citep{ad20} and the equivalence between robust and private estimation~\citep{hkmn23,auz23}.
Our lower bounds refine the coupling method of \citet{asz21}.

\paragraph{Pure versus approximate privacy.}
\citet{almm19} and \citet{blm20} show that approximate-DP learnability is equivalent to online learnability, while pure DP needs more.
At the level of rates, approximate DP often keeps non-private rates while pure DP can cost much more~\citep{smith11,fx15,bnsv15}.
Separations of this kind exist for learning~\citep{bns13}, counting queries~\citep{buv14}, and marginals~\citep{su17}.
Language learning gives a new instance with a sharp answer.
Our rates are universal rates in the sense of \citet{bhmvy21}: each holds for every distribution, with constants that depend on it.
With a public witness bound, pure DP costs at most a constant factor in the exponent, and on hard instances this factor is necessary.
Without one, it separates from approximate DP at the level of universal rates (Theorem~\ref{thm:gen-nec}).

\paragraph{Comparison with other tasks.}
Table~\ref{tab:compare} in Appendix~\ref{app:compare} puts our prices next to those of other private learning tasks, and it separates sample complexities from error rates.
Identification behaves like private hypothesis selection~\citep{bksw19}, where privacy adds only a lower-order term once $\eps$ exceeds the accuracy.
Generation without a public bound is a rate analog of the separations between pure and approximate DP in PAC learning~\citep{fx15,almm19}.

\section{Setup and three master lemmas}\label{sec:setup}

\paragraph{Notation.}
We write $\N=\{1,2,\dots\}$, $[n]=\{1,\dots,n\}$, and $a_+=\max\{a,0\}$.
For a distribution $\cD$ on $\cU$, $\supp(\cD)=\{x:\cD(x)>0\}$.
For a sample $S=(x_1,\dots,x_n)\in\cU^n$ and a string $x$, the count $N_S(x)$ is the number of indices $t$ with $x_t=x$.
The Hamming distance $\dH(S,S')$ is the number of indices $t$ with $x_t\ne x'_t$.
Constants in $O,\Omega,\Theta$ may depend on $\cD$ and $\cC$, but not on $n$, $\eps$, or $\delta$.

\paragraph{The two tasks.}
The risk of a language $L$ is $\err_\cD(L)=\Pr_{x\sim\cD}[x\notin L]$, and its empirical risk is $\err_S(L)=\frac1n\sum_{t=1}^n\one\{x_t\notin L\}$.
An identification algorithm $\cA$ maps a sample to a language in $\cC$, and a generation algorithm maps it to a string in $\cU$.
Both can use internal randomness.
Their errors are
\begin{align*}
\iderr(\cA,\cD,\cC,n)&=\E\bigl[\err_\cD(\cA(S))\bigr]-\inf_{L\in\cC}\err_\cD(L),\\
\generr(\cA,\cD,\cC,n)&=\Pr\bigl[\cA(S)\notin\supp(\cD)\setminus S\bigr],
\end{align*}
where $S\sim\cD^n$ and the randomness of $\cA$ is independent of $S$.
The final version of \citet{hp26} calls this identification objective \emph{weak identification}.

\paragraph{Differential privacy.}
Two samples $S,S'\in\cU^n$ are \emph{neighbors} if $\dH(S,S')=1$.

\begin{definition}[Differential privacy]\label{def:dp}
Let $\eps>0$ and $\delta\ge0$.
A randomized algorithm $\cA:\cU^n\to\cR$ is \emph{$(\eps,\delta)$-DP}~\citep{dmns06} if $\Pr[\cA(S)\in F]\le e^{\eps}\Pr[\cA(S')\in F]+\delta$ for all neighbors $S,S'$ and every measurable $F\subseteq\cR$.
It is \emph{$\eps$-DP} (pure DP) if $\delta=0$.
We call every algorithm $\infty$-DP.
\end{definition}

An $\eps$-DP algorithm satisfies $\Pr[\cA(S)\in F]\le e^{\eps\,\dH(S,S')}\Pr[\cA(S')\in F]$ for all $S,S'$ (group privacy), and post-processing keeps $(\eps,\delta)$-DP~\citep{dr14}.

\paragraph{Scores and two mechanisms.}
Let $\cR$ be a finite set of candidates.
A \emph{score} is a map $q:\cU^n\times\cR\to\R$, and its \emph{sensitivity} is the largest change $|q(S,c)-q(S',c)|$ over all neighbors $S,S'$ and all $c\in\cR$.
The mechanisms below take a public \emph{sensitivity bound} $\Delta>0$, a number that is at least the sensitivity of $q$; we say that $q$ has sensitivity at most $\Delta$.
The exponential mechanism $\EM_\eps(S,q)$ outputs $c$ with probability proportional to $\exp(\eps\,q(S,c)/(2\Delta))$~\citep{mt07}.
The Gaussian noisy argmax $\GN_{\eps,\delta}(S,q)$ outputs a maximizer of $q(S,c)+Z_c$, where the $Z_c$ are independent $\cN(0,\sigma^2)$ variables with $\sigma=\Delta\sqrt{|\cR|}\,\sqrt{2\log(1.25/\delta)}/\eps$.
The next lemma states the only property of these mechanisms that we use.

\begin{lemma}[Private selection with a margin]\label{lem:selection}
Let $q$ be a score with sensitivity at most $\Delta$, where $\Delta>0$.
\begin{enumerate}[label=(\alph*)]
\item The mechanism $\EM_\eps(\cdot,q)$ is $\eps$-DP. For $\eps,\delta\in(0,1)$, the mechanism $\GN_{\eps,\delta}(\cdot,q)$ is $(\eps,\delta)$-DP.
\item Fix a sample $S$ and a set $\cG\subseteq\cR$ of good candidates.
Assume that the \emph{margin} $\Gamma=\max_{c\in\cG}q(S,c)-\max_{c\notin\cG}q(S,c)$ is positive.
Then
\[
\Pr\bigl[\EM_\eps(S,q)\notin\cG\bigr]\le|\cR|\,e^{-\eps\Gamma/(2\Delta)}
\quad\text{and}\quad
\Pr\bigl[\GN_{\eps,\delta}(S,q)\notin\cG\bigr]\le2|\cR|\,e^{-\Gamma^2/(8\sigma^2)}.
\]
\end{enumerate}
\end{lemma}

\begin{proof}[Proof sketch]
Part (a) is the guarantee of the two mechanisms~\citep{mt07,dkm+06}, because the score vector has $\ell_2$-sensitivity at most $\Delta\sqrt{|\cR|}$.
For part (b), each bad candidate has probability at most $e^{-\eps\Gamma/(2\Delta)}$ relative to the best good candidate.
Gaussian noisy argmax selects a good candidate whenever every $|Z_c|$ is less than $\Gamma/2$.
Appendix~\ref{app:master} gives the full proof.
\end{proof}

The classical Gaussian mechanism needs $\eps<1$.
An $(\eps_0,\delta)$-DP algorithm is $(\eps,\delta)$-DP for every $\eps\ge\eps_0$, so for $\eps\ge1$ we run the Gaussian versions with $\eps_0=1/2$, and their bounds hold with $1/2$ in place of $\eps$.

\begin{algorithm}[t]
\caption{Private selection by a score, and the two learners}
\label{alg:select}
\begin{algorithmic}[1]
\REQUIRE sample $S\in\cU^n$; finite candidate set $\cR$; score $q$ with sensitivity at most $\Delta$; $\eps>0$ (and $\delta\in(0,1)$, $\eps<1$ for the Gaussian version)
\IF{pure DP}
\STATE sample $\wh c\in\cR$ with $\Pr[\wh c=c]\propto\exp\bigl(\eps\,q(S,c)/(2\Delta)\bigr)$ \hfill $\triangleright$ $\EM_\eps(S,q)$
\ELSE
\STATE draw $Z_c\sim\cN(0,\sigma^2)$ independently, with $\sigma=\Delta\sqrt{|\cR|}\sqrt{2\log(1.25/\delta)}/\eps$
\STATE $\wh c\gets$ a maximizer of $q(S,c)+Z_c$ over $c\in\cR$ \hfill $\triangleright$ $\GN_{\eps,\delta}(S,q)$
\ENDIF
\STATE \textbf{identification} (Section~\ref{sec:id}): $\cR=[f]$, $q=\qid$, $\Delta=2/n$; output $L_{\wh c}$
\STATE \textbf{generation} (Section~\ref{sec:gen}): $\cR=[f]$, $q=\qgen$, $\Delta=1$; draw $j\sim\Unif([2^n])$ and output the $j$-th string of $L_{\wh c}$
\end{algorithmic}
\end{algorithm}

One changed entry of $S$ moves each score by at most $\Delta$, so it closes the margin by at most $2\Delta$.
Thus $\rho=\Gamma/(2\Delta)$ is, up to rounding, the number of entries that an adversary must change to close the margin.
Pure DP converts it into the failure probability $e^{-\eps\rho}$.
Gaussian noise has quadratic tails, so approximate DP converts it into $e^{-\Theta(\eps^2\Gamma^2/(\Delta^2|\cR|\log(1/\delta)))}$.
Every upper bound in this paper is Lemma~\ref{lem:selection} applied to one score in Algorithm~\ref{alg:select}.

The second lemma is the converse.
It is a one-line refinement of the coupling method of \citet{asz21}.

\begin{lemma}[Private two-point inequality]\label{lem:lecam}
Let $\eps\in(0,\infty]$ and let $\cA$ be $\eps$-DP.
Let $(S,S')$ be a coupling of two distributions on $\cU^n$, and let $F$ be a set of outputs.
Then
\[
\Pr[\cA(S)\notin F]+\Pr[\cA(S')\in F]\;\ge\;\E\bigl[e^{-\eps\,\dH(S,S')}\bigr],
\]
with $e^{-\infty\cdot0}=1$.
If the pairs $(x_t,x'_t)$ are i.i.d.\ and $\Pr[x_t\ne x'_t]=p$, the right side equals $\bigl(1-p\,(1-e^{-\eps})\bigr)^n$.
\end{lemma}

\begin{proof}
Fix a pair $(S,S')$, and let $\varphi$ and $\varphi'$ be the probabilities, over the coins of $\cA$, that $\cA(S)\in F$ and $\cA(S')\in F$.
By group privacy (Lemma~\ref{lem:postproc}), $\varphi'\ge e^{-\eps\dH(S,S')}\varphi\ge e^{-\eps\dH(S,S')}-(1-\varphi)$.
Take expectations over the coupling.
In the i.i.d.\ case, $\dH(S,S')$ has the binomial law with parameters $n$ and $p$, and $\E[e^{-\eps\dH(S,S')}]$ is its moment generating function at $-\eps$.
\end{proof}

At $\eps=\infty$, the right side is $\Pr[S=S']$.
For an optimal coupling, this probability is one minus the total variation distance, so Lemma~\ref{lem:lecam} contains Le~Cam's two-point bound~\citep{lecam73,yu97}.
For finite $\eps$, the bound has a simple reading.
Each coordinate on which the two samples differ ``survives'' with probability $\theta_\eps=1-e^{-\eps}$, and $(1-p\theta_\eps)^n$ is the probability that no coordinate survives.
Pure privacy thus acts like a thinning of the informative samples, and
\begin{equation}\label{eq:theta}
(1-e^{-1})\min\{1,\eps\}\;\le\;\theta_\eps\;\le\;\min\{1,\eps\}.
\end{equation}
This is the source of the factor $\min\{1,\eps\}$ in Table~\ref{tab:summary}; the two regimes $\eps<1$ and $\eps\ge1$ need no separate argument.

Lemma~\ref{lem:lecam} bounds a sum of two errors.
The third lemma transfers the probability of one event between two distributions that share most of their mass.
Universal lower bounds need this form, because they compare one fixed distribution with a whole sequence of nearby ones.

\begin{lemma}[Private resampling inequality]\label{lem:resampling}
Let $P,R_0,R_1$ be distributions on $\cU$, let $t\in[0,1]$, and let $\cD_b=(1-t)P+tR_b$ for $b\in\{0,1\}$.
Let $\cA$ be $\eps$-DP, and let $F$ be a set of outputs.
Then
\[
\Pr_{S\sim\cD_1^n}[\cA(S)\in F]\;\le\;\kappa_{\eps,t}^{\,n}\,\Pr_{S\sim\cD_0^n}[\cA(S)\in F],
\qquad
\kappa_{\eps,t}=1+t\,(e^\eps-1)\;\le\;e^{t(e^\eps-1)}.
\]
\end{lemma}

\begin{proof}
For $0\le j\le n$, let $Q_j$ be the law of a sample whose first $j$ entries come from $\cD_1$ and whose other entries come from $\cD_0$.
Then $Q_0=\cD_0^n$ and $Q_n=\cD_1^n$.
Fix $j\ge1$ and every entry except the $j$-th.
Let $\psi(x)$ be the probability, over the coins of $\cA$, that $\cA$ outputs a point of $F$ when the $j$-th entry is $x$.
By the definition of pure DP, $\sup_x\psi(x)\le e^\eps\inf_x\psi(x)$.
Because $\cD_1-\cD_0=t(R_1-R_0)$, we get $\E_{x\sim\cD_1}\psi-\E_{x\sim\cD_0}\psi\le t(\sup\psi-\inf\psi)\le t(e^\eps-1)\E_{x\sim\cD_0}\psi$.
Averaging over the other entries gives $Q_j(F)\le\kappa_{\eps,t}\,Q_{j-1}(F)$, and induction on $j$ finishes the proof.
\end{proof}

The factor $\kappa_{\eps,t}$ is sharp.
Let $R_1$ be a point mass on a string that $P$ and $R_0$ never produce.
Let $N$ be the number of copies of that string in $S$.
The algorithm that reports $F$ with probability $(1-e^{-\eps})e^{-\eps(n-N)}$ is $\eps$-DP and attains equality.
At $\eps=\infty$ the bound is empty, as it must be: a non-private algorithm can react to one draw from $R_1$.
For finite $\eps$, pure DP cannot see the absence of a string.
If $R_1$ charges strings that $\cD_0$ never produces, an event that is likely under $\cD_1$ still has probability at least about $e^{-tn(e^\eps-1)}$ under $\cD_0$.

\section{Private language identification}\label{sec:id}

We use the assumption under which \citet{hp26} prove the non-private rate.

\begin{assumption}[Attained optimum]\label{asm:attained}
Some index $i$ satisfies $\err_\cD(L_i)=\inf_{L\in\cC}\err_\cD(L)$.
We call the smallest such index the \emph{target index} $i^\star$.
\end{assumption}

In the realizable case, some language has risk $0$, so the assumption holds and all results of this section apply.
If $i^\star>1$, the earlier languages are strictly worse, and we write $\cgap=\min_{j<i^\star}\bigl(\err_\cD(L_j)-\err_\cD(L_{i^\star})\bigr)>0$; if $i^\star=1$, we set $\cgap=1$.

\paragraph{The non-private rule.}
Fix a horizon $f=f(n)$ with $f(n)\to\infty$.
The \emph{empirical margin} of an index $i\in[f]$ is $M_S(1)=1$ and $M_S(i)=\min_{j<i}\bigl(\err_S(L_j)-\err_S(L_i)\bigr)$ for $i\ge2$.
The rule of \citet{hp26} outputs the largest $i\in[f]$ that passes the \emph{margin test} $M_S(i)>2/f$.
This rule is not private: one sample can move an index across the threshold.
We replace the yes/no test by the amount by which the output statement holds.

\begin{definition}[Signed-margin score]\label{def:qid}
For $i\in[f]$, let
\[
\qid(S,i)=\min\Bigl\{M_S(i)-\frac2f,\ \min_{i<k\le f}\Bigl(\frac2f-M_S(k)\Bigr)\Bigr\},
\]
where a minimum over an empty set is $+\infty$.
\end{definition}

The score is the margin by which the statement ``$i$ passes the test and every later index fails it'' holds.
Thus it is nonnegative at the output of the non-private rule and nonpositive at every other index (Proposition~\ref{prop:limits}).

\begin{lemma}[Signed-margin score]\label{lem:qid}
\begin{enumerate}[label=(\alph*)]
\item The score $\qid$ has sensitivity at most $2/n$.
\item Let $\cE$ be the event that $|\err_S(L_i)-\err_\cD(L_i)|\le1/(4f)$ for every $i\in[f]$. Then $\Pr[\cE^c]\le2f\,e^{-n/(8f^2)}$.
\item Assume $f\ge\max\{i^\star,3/\cgap,3\}$. On $\cE$, the target index has margin $\qid(S,i^\star)-\max_{i\ne i^\star}\qid(S,i)\ge1/f$.
\end{enumerate}
\end{lemma}

\begin{proof}[Proof sketch]
One sample changes each empirical risk by at most $1/n$ and each difference in $M_S$ by at most $2/n$, and minima keep this bound.
Part (b) is Hoeffding's inequality with a union bound.
For part (c), on $\cE$ the target index passes the test with slack $1/(2f)$, because $\cgap\ge3/f$.
Every later index fails with slack $3/(2f)$, because it is not better than $L_{i^\star}$; so $\qid(S,i^\star)\ge1/(2f)$.
Every earlier index has $\qid\le2/f-M_S(i^\star)\le-1/(2f)$, and every later index has $\qid\le M_S(i)-2/f\le-3/(2f)$.
Appendix~\ref{app:id} gives the complete proof.
\end{proof}

\paragraph{The algorithm.}
The private identifier is Algorithm~\ref{alg:select} with the score $\qid$ and the candidates $\cR=[f]$.
It selects $\wh i$ by $\EM_\eps$ (pure DP) or by $\GN_{\eps,\delta}$ (approximate DP), and it outputs $L_{\wh i}$.
Lemmas~\ref{lem:selection} and~\ref{lem:qid} give the next theorem at once, with margin $\Gamma=1/f$, sensitivity $\Delta=2/n$, and $f$ candidates.

\begin{theorem}[Private identification]\label{thm:id}
Let Assumption~\ref{asm:attained} hold, and let $f(n)\to\infty$.
For every $n$ with $f(n)\ge\max\{i^\star,3/\cgap,3\}$:
\begin{enumerate}[label=(\alph*)]
\item for every $\eps>0$, the exponential-mechanism identifier is $\eps$-DP and
\[
\iderr\le2f\,e^{-n/(8f^2)}+f\,e^{-\eps n/(4f)};
\]
\item for every $\eps,\delta\in(0,1)$, the Gaussian identifier is $(\eps,\delta)$-DP and
\[
\iderr\le2f\,e^{-n/(8f^2)}+2f\exp\Bigl(-\frac{\eps^2n^2}{64f^3\log(1.25/\delta)}\Bigr).
\]
\end{enumerate}
\end{theorem}

The first term is the non-private error of the margin rule, and the second term is the price of privacy.
The privacy term is at most the first term if $\eps\ge1/(2f)$ in part~(a), and if $\eps^2n\ge8f\log(1.25/\delta)$ in part~(b).
For fixed $\eps$, both conditions hold for large $n$ (in part~(b) if, for example, $f\le\sqrt n$ and $\delta\ge n^{-O(1)}$).

\begin{corollary}[Identification rates]\label{cor:id-rates}
Let $r(n)=o(n)$ with $r(n)\ge\log n$, and let $f(n)=\lfloor\sqrt{n/(16r(n))}\rfloor$.
For every $n$ with $f(n)\ge\max\{i^\star,3/\cgap,3\}$, which holds for all large $n$, $\iderr\le4e^{-r(n)}$ for the pure identifier if $\eps\ge2\sqrt{r(n)/n}$, and for the Gaussian identifier with $\eps,\delta\in(0,1)$ if $\eps^2\sqrt{n\,r(n)}\ge2\log(1.25/\delta)$.
\end{corollary}

So both private identifiers achieve every non-private rate $e^{-r(n)}$ at every fixed $\eps$, and the condition on $\delta$ allows every $\delta\ge n^{-O(1)}$.
If $r(n)<\log n$, apply the corollary to $\bar r(n)=\max\{r(n),\log n\}$, with the horizon $f(n)=\lfloor\sqrt{n/(16\bar r(n))}\rfloor$.
Because $e^{-\bar r(n)}\le e^{-r(n)}$, the bound $\iderr\le4e^{-r(n)}$ still holds, but the privacy conditions become $\eps\ge2\sqrt{\bar r(n)/n}$ and $\eps^2\sqrt{n\,\bar r(n)}\ge2\log(1.25/\delta)$, which are stronger than the conditions stated with $r(n)$.
At every fixed $\eps$ and every $\delta\ge n^{-O(1)}$ they still hold for all large $n$, so both private identifiers achieve the rate $e^{-r(n)}$ for every $r(n)=o(n)$, including $r(n)<\log n$.
Our bounds show a cost of privacy only when $\eps$ decreases with $n$: for pure DP, below $2\sqrt{\bar r(n)/n}$, which is $2\sqrt{r(n)/n}$ when $r(n)\ge\log n$.
On collections with the IPP condition, Theorem~\ref{thm:id-lb} shows that some cost is unavoidable once $\eps\le2r(n)/n$, because then $\frac18e^{-\eps n/3}>4e^{-r(n)}$ for large $n$.

\begin{remark}[Why the signed margin matters]\label{rem:index-reward}
A natural alternative score rewards large indices directly, for example $i-f^2\bigl(2/f-M_S(i)\bigr)_+$.
The score needs the weight $f^2$ to beat the index reward $i\le f$.
This weight raises the sensitivity to $2f^2/n$, while the margin stays $1$.
Lemma~\ref{lem:selection} then gives the privacy term $f\,e^{-\eps n/(4f^2)}$, which costs a factor $2\eps$ in the exponent for every $\eps<1/2$.
The signed margin puts the preference for large indices inside the minimum over $k>i$, so it pays no index reward.
\end{remark}

\section{Private language generation}\label{sec:gen}

Fix an enumeration $\cU=\{u_1,u_2,\dots\}$, and let $\cU_\ell=\{u_1,\dots,u_\ell\}$ be its first $\ell$ strings; we call $\ell$ a \emph{level}.
Assume that every language in $\cC$ is infinite.
A language $L_i$ is \emph{good} if $L_i\subseteq\supp(\cD)$ and \emph{bad} otherwise.
A good language is a safe source of valid strings, and a bad language contains a string outside the support.
The \emph{witness index} $w=w(\cC,\cD)$ is the smallest level $w\in\N$ such that every bad language contains a string of $\cU_w$ outside $\supp(\cD)$.
If there is no such level, then $w=\infty$.
We use the two assumptions of \citet{hp26}, under which the non-private rate is $e^{-\Theta(n)}$.

\begin{assumption}[Generation]\label{asm:gen}
(a) The witness index is finite.
(b) Some language in $\cC$ is good; the \emph{target index} $i^\star$ is the smallest such index.
\end{assumption}

\paragraph{Prefix counts and capped-depth scores.}
The \emph{prefix count} of a language $L_i$ at level $\ell$ is the smallest count of its first strings,
\[
a_{i,\ell}(S)=\min_{u\in L_i\cap\cU_\ell}N_S(u),
\]
with $a_{i,\ell}(S)=n$ if $L_i\cap\cU_\ell=\varnothing$.
The prefix count does not increase with $\ell$.
The non-private rule of \citet{hp26} prefers the language whose prefix stays covered, $a_{i,\ell}(S)\ge1$, up to the largest level $\ell$.
A bad language loses this race at its first witness, which never appears in $S$.
But one sample can hide or insert a string, so the race needs a cap that converts depth into counts.

\begin{definition}[Capped-depth score]\label{def:qgen}
A \emph{cap profile} is a nondecreasing map $\lambda:\N\cup\{0\}\to[0,\infty]$ with $\lambda_0=0$.
The capped-depth score of $L_i$ is
\[
\qgen(S,i)=\max_{\ell\in\N}\ \min\bigl\{a_{i,\ell}(S),\,\lambda_\ell\bigr\}.
\]
\end{definition}

The maximum exists: $L_i$ is infinite, so $a_{i,\ell}(S)=0$ for all large $\ell$.
The score is high if many samples cover the prefix of $L_i$ up to a deep level.
The profile states how much count each level is worth; Figure~\ref{fig:capped} shows the two profiles that we use.

\begin{figure}[t]
\centering
\includegraphics[width=\linewidth]{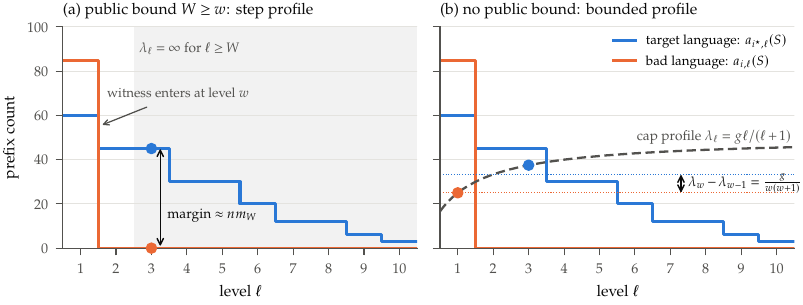}
\caption{How the capped-depth score works (illustrative counts, $n=100$, witness index $w=2$).
Blue: the target language, whose prefix count decreases slowly. Orange: a bad language, whose prefix count drops to $0$ at its witness.
Dots mark the scores $\max_\ell\min\{a_{i,\ell}(S),\lambda_\ell\}$.
(a) With a public bound $W$, the step profile reads off the prefix count at level $W$, and the margin is about $nm_W$.
(b) Without one, the bounded profile $\lambda_\ell=g\ell/(\ell+1)$ caps every level: the bad language scores at most $\lambda_{w-1}$, the target at least $\lambda_w$, and the margin is only $g/(w(w+1))$.}
\label{fig:capped}
\end{figure}

\begin{lemma}[Capped-depth score]\label{lem:qgen}
Let $\lambda$ be a cap profile, and let $S$ be a sample with all entries in $\supp(\cD)$.
\begin{enumerate}[label=(\alph*)]
\item The score $\qgen$ has sensitivity at most $1$.
\item Every bad language $L_i$ has $\qgen(S,i)\le\lambda_{w-1}$.
\item For every level $\ell\in\N$, the target language has $\qgen(S,i^\star)\ge\min\{a_{i^\star,\ell}(S),\lambda_\ell\}$.
\end{enumerate}
\end{lemma}

\begin{proof}
(a) One sample changes each count by at most $1$, and minima and maxima keep this bound.
(b) Let $u_k\notin\supp(\cD)$ be the first string of $L_i$ outside the support, so $k\le w$.
For $\ell\ge k$, the string $u_k$ is absent from $S$, so $a_{i,\ell}(S)=0$.
For $\ell<k$, the term is at most $\lambda_\ell\le\lambda_{w-1}$.
(c) The maximum is at least its term at level $\ell$.
\end{proof}

\paragraph{The algorithm.}
Fix a horizon $f=f(n)$ and a cap profile $\lambda$; the theorems need only $f\ge i^\star$, so we take $f(n)\to\infty$ when $i^\star$ is unknown.
The private generator is Algorithm~\ref{alg:select} with the score $\qgen$: it selects $\wh i$ by $\EM_\eps(S,\qgen)$ or by $\GN_{\eps,\delta}(S,\qgen)$ over $\cR=[f]$.
Then it draws $j$ uniformly from $[2^n]$, independently of $S$, and outputs the $j$-th string of $L_{\wh i}$.
If $L_{\wh i}$ is good, the output is valid, and it lies in $S$ with probability at most $n2^{-n}$.
The output step is post-processing, so the generator has the privacy of the selection step.
One number of the distribution governs all rates below: the \emph{mass floor} at level $\ell$,
\[
m_\ell=\min_{u\in L_{i^\star}\cap\,\cU_\ell}\cD(u)\qquad(\text{and }m_\ell=1\text{ if }L_{i^\star}\cap\cU_\ell=\varnothing),
\]
the smallest probability of a target string among the first $\ell$ strings.
It is positive because the target language is good, and it does not increase with $\ell$.

\paragraph{With a public witness bound.}
Assume that the learner knows a level $W\ge w$ before it sees $S$.
The profile with $\lambda_\ell=0$ for $\ell<W$ and $\lambda_\ell=\infty$ for $\ell\ge W$ gives the plain prefix count $\qgen(S,i)=a_{i,W}(S)$.
By Lemma~\ref{lem:qgen}, every bad language has score $0$, and the target language has score about $nm_W$.
The margin is linear in $n$, and the learner needs no threshold and no knowledge of the masses.

\begin{theorem}[Generation with a public witness bound]\label{thm:gen-public}
Let Assumption~\ref{asm:gen} hold, and let $W\ge w$ be public.
For every $n$ with $f(n)\ge i^\star$:
\begin{enumerate}[label=(\alph*)]
\item for every $\eps>0$, the exponential-mechanism generator is $\eps$-DP and
\[
\generr\le W e^{-nm_W/8}+f\,e^{-\eps nm_W/4}+n2^{-n};
\]
\item for every $\eps,\delta\in(0,1)$, the Gaussian generator is $(\eps,\delta)$-DP and
\[
\generr\le W e^{-nm_W/8}+2f\exp\Bigl(-\frac{\eps^2n^2m_W^2}{64f\log(1.25/\delta)}\Bigr)+n2^{-n}.
\]
\end{enumerate}
\end{theorem}

\begin{proof}[Proof sketch]
By the Chernoff bound, every target string in $\cU_W$ appears at least $nm_W/2$ times, except with probability $We^{-nm_W/8}$.
On this event, the margin over the bad languages is $\Gamma\ge nm_W/2$, with sensitivity $\Delta=1$ (Lemma~\ref{lem:qgen}).
Lemma~\ref{lem:selection} bounds the probability of a bad selection, and $n2^{-n}$ bounds the probability of a repeated string.
Appendix~\ref{app:gen} gives the complete proof.
\end{proof}

For $n\ge14$, we have $n2^{-n}\le e^{-n/2}$, so part (a) gives $\generr\le(W+f+1)\,e^{-\min\{1,2\eps\}nm_W/8}$.
This is the rate $\exp(-\Omega(\min\{1,\eps\}n))$ whenever $\log f(n)=o(\min\{1,\eps\}\,n)$, and Section~\ref{sec:lb} shows that the factor $\min\{1,\eps\}$ cannot be removed in general (Example~\ref{ex:sharp}).
In part~(b), the privacy exponent is at least the sampling exponent whenever $\eps^2nm_W\ge8f\log(1.25/\delta)$. So approximate DP keeps the non-private rate.

\paragraph{Without a public witness bound.}
Now the level $w$ is unknown.
For a threshold $g=g(n)$, the profile $\lambda_\ell=g\,\ell/(\ell+1)$ stays below $g$ and increases at every level.
By Lemma~\ref{lem:qgen}, the target language beats every bad language by $\lambda_w-\lambda_{w-1}=g/(w(w+1))$ once its strings in $\cU_w$ appear $g$ times each.

\begin{theorem}[Generation without a public witness bound]\label{thm:gen-nopublic}
Let Assumption~\ref{asm:gen} hold, and use the profile $\lambda_\ell=g\,\ell/(\ell+1)$.
For every $n$ with $f(n)\ge i^\star$ and $0<g(n)\le nm_w/2$:
\begin{enumerate}[label=(\alph*)]
\item for every $\eps>0$, the exponential-mechanism generator is $\eps$-DP and
\[
\generr\le w e^{-nm_w/8}+f\exp\Bigl(-\frac{\eps g}{2w(w+1)}\Bigr)+n2^{-n};
\]
\item for every $\eps,\delta\in(0,1)$, the Gaussian generator is $(\eps,\delta)$-DP and
\[
\generr\le w e^{-nm_w/8}+2f\exp\Bigl(-\frac{\eps^2g^2}{16\,w^2(w+1)^2f\log(1.25/\delta)}\Bigr)+n2^{-n}.
\]
\end{enumerate}
\end{theorem}

The condition $g(n)\le nm_w/2$ involves the unknown mass floor $m_w$.
It holds for all large $n$ if $g(n)=\lfloor n/\psi(n)\rfloor$ for some $\psi(n)\to\infty$ with $\psi(n)=o(n)$.

\begin{corollary}[Generation rates without a public bound]\label{cor:gen-rates}
Let Assumption~\ref{asm:gen} hold.
\begin{enumerate}[label=(\alph*)]
\item Let $\log f(n)=o(\min\{1,\eps\}\,n)$. If the learner knows a lower bound $\underline m\in(0,m_w]$ on the mass floor, then $g=\lfloor n\underline m/2\rfloor$ gives $\generr\le e^{-\Omega(\min\{1,\eps\}n)}$ under pure DP.
\item With $g(n)=\lfloor n/\psi(n)\rfloor$, $\psi(n)\to\infty$, $\psi(n)=o(n)$, and $\log f(n)=o(\min\{1,\eps\}\,n/\psi(n))$, pure DP gives $\generr\le e^{-\Omega(\min\{1,\eps\}\,n/\psi(n))}$. With $\psi=n/r$, this is every rate $e^{-\Omega(\min\{1,\eps\}\,r(n))}$ with $r(n)\to\infty$ and $r(n)=o(n)$.
\item Let $\log f(n)=o(n)$. With $g(n)=\lfloor n/\psi(n)\rfloor$ and $\eps,\delta\in(0,1)$, approximate DP gives $\generr\le e^{-\Omega(n)}$ if $\eps^2n\ge\psi(n)^3f(n)\log(1.25/\delta)$ for all large $n$. For fixed $\eps$, slowly growing $\psi$ and $f$ satisfy this condition whenever $\log(1/\delta)=o(n)$, for example if $\delta\ge n^{-O(1)}$.
\end{enumerate}
\end{corollary}

\paragraph{The price of an unknown witness index.}
Corollary~\ref{cor:gen-rates}(b) stops short of an exponential rate.
The next theorem shows that no pure-DP generator does better, even with two languages.
In the same setting, the non-private rule of \citet{hp26} and the approximate-DP generator of Corollary~\ref{cor:gen-rates}(c) are exponential.

\begin{theorem}[Pure DP needs a public bound]\label{thm:gen-nec}
Let $\cC$ be finite, and let it contain two infinite languages $L,L'$ with $L\cap L'$ finite.
Let $\eps\in(0,\infty)$, and let $\cA$ be an $\eps$-DP generator whose error tends to $0$ for every distribution that satisfies Assumption~\ref{asm:gen}.
Then every distribution $\cD$ with $\supp(\cD)=L$ satisfies, for every $c>0$,
\[
\generr(\cA,\cD,\cC,n)\;\ge\;\tfrac12\,e^{-cn}\qquad\text{for all large }n.
\]
The same holds for $(\eps,\delta_n)$-DP generators if $\delta_n\le e^{-c_0n}$ for some $c_0>0$.
\end{theorem}

\begin{proof}[Proof sketch]
Let $t_\ell=\cD(L\setminus\cU_\ell)$ be the $\cD$-mass of $L$ beyond level $\ell$, so $t_\ell\to0$.
Write $\cD=(1-t_\ell)P_\ell+t_\ell R_\ell$, where $P_\ell$ lives on $L\cap\cU_\ell$, and let $\cD'_\ell=(1-t_\ell)P_\ell+t_\ell R'$ for a distribution $R'$ with support $L'$.
Under $\cD'_\ell$, the language $L'$ is good and Assumption~\ref{asm:gen} holds, but the witness index tends to $\infty$ with $\ell$.
Once $S$ contains the finitely many valid strings outside $L'\setminus L$, every valid new output under $\cD'_\ell$ lies in $L'\setminus L$. Every such output is invalid under $\cD$.
Lemma~\ref{lem:resampling} transfers the probability of this set from $\cD'_\ell$ to $\cD$ at the price $\kappa_{\eps,t_\ell}^{\,n}$, and $\kappa_{\eps,t_\ell}\to1$.
Figure~\ref{fig:hardpair} shows the construction, and Appendix~\ref{app:gen} gives the complete proof.
\end{proof}

\begin{figure}[t]
\centering
\includegraphics[width=\linewidth]{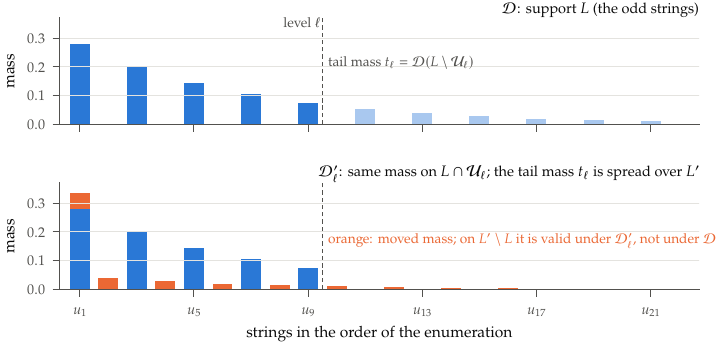}
\caption{The hard pair of Theorem~\ref{thm:gen-nec}, for $L$ the odd strings and $L'=\{u_1\}$ together with the even strings.
The distribution $\cD'_\ell$ agrees with $\cD$ on the first $\ell$ strings and spreads the tail mass $t_\ell$ of $L$ over $L'$.
Once the strings up to level $\ell$ have appeared, a valid new output under $\cD'_\ell$ lies in $L'\setminus L$, and it is invalid under $\cD$.
Both distributions mix a common part with weight $1-t_\ell$, so by Lemma~\ref{lem:resampling} every event keeps its probability up to the factor $(1+t_\ell(e^\eps-1))^n$; and $t_\ell\to0$ as $\ell\to\infty$.}
\label{fig:hardpair}
\end{figure}

The hardness is not a pathology: it holds for every distribution supported on $L$.
A public bound $W$ rules out the distributions $\cD'_\ell$, and the same $\cD$ then has the rate $e^{-\Theta(\min\{1,\eps\}n)}$ (Theorem~\ref{thm:gen-public}).
A known lower bound on the mass floor rules them out too, because their mass floor is at most $t_\ell$.

\begin{remark}[Realizable distributions]\label{rem:realizable}
If $\supp(\cD)$ is a language of $\cC$, Assumption~\ref{asm:gen}(b) holds, but (a) may fail, and our upper bounds need it.
The lower bound survives realizability.
Replace $\cC$ by the countable collection that consists of $L$ and the languages $(L\cap\cU_\ell)\cup L'$ for $\ell\in\N$.
Then $\cD$ and every $\cD'_\ell$ of the proof are realizable and satisfy Assumption~\ref{asm:gen}, and the same proof applies (Appendix~\ref{app:gen}).
\end{remark}

\begin{corollary}[A sharp threshold]\label{cor:gen-threshold}
Let $\cC$ be as in Theorem~\ref{thm:gen-nec}.
\begin{enumerate}[label=(\alph*)]
\item Without a public bound, pure DP achieves $e^{-r(n)}$ for every $r(n)=o(n)$ (Corollary~\ref{cor:gen-rates}(b)), and no exponential rate.
\item Without a public bound, $(\eps,\delta_n)$-DP achieves an exponential rate if and only if $\log(1/\delta_n)=o(n)$.
\end{enumerate}
\end{corollary}

So, for these collections, a public witness bound is worth an exponential under pure DP and under approximate DP with $\delta_n\le e^{-c_0n}$.
It is worth nothing if $\log(1/\delta_n)=o(n)$, and nothing without privacy.

\section{Lower bounds}\label{sec:lb}

Both lower bounds apply Lemma~\ref{lem:lecam} to two distributions that put mass $3/4$ on a common string.
A coupled pair of samples then differs in each coordinate with probability $1/4$, and the bound is $(1-\theta_\eps/4)^n$ with $\theta_\eps=1-e^{-\eps}$.
We use $1-x\ge e^{-4x/3}$ for $x\in[0,1/4]$ to write it as an exponential.
The two arguments have different shapes.
Identification is symmetric: each distribution has its own optimal language, and the coupling bounds the two misidentification probabilities together.
Generation is one-sided: success under one distribution forces outputs that fail under the other.
Theorem~\ref{thm:gen-nec} uses a third, multiplicative argument (Lemma~\ref{lem:resampling}), which compares one distribution with a whole sequence of distributions.

\begin{definition}[IPP condition]\label{def:ipp}
A collection $\cC$ satisfies the \emph{intersecting private-pair} (IPP) condition if it contains two languages $L\ne L'$ with $L\cap L'\ne\varnothing$, and each of them contains a \emph{private} string.
A string is private to $L$ if no other language of $\cC$ contains it.
\end{definition}

\begin{theorem}[Identification lower bound]\label{thm:id-lb}
Let $\cC$ satisfy the IPP condition.
There are two distributions $\cD_0,\cD_1$ that satisfy Assumption~\ref{asm:attained}, such that for every $\eps\in(0,\infty]$, every $\eps$-DP identification algorithm $\cA$, and every $n\ge1$,
\[
\max_{b\in\{0,1\}}\iderr(\cA,\cD_b,\cC,n)\;\ge\;\frac18\Bigl(1-\frac{\theta_\eps}4\Bigr)^n\;\ge\;\frac18\,e^{-\min\{1,\eps\}\,n/3}.
\]
\end{theorem}

\begin{proof}
Let $s_0\in L\cap L'$, let $s_1$ be private to $L$, and let $s_2$ be private to $L'$.
Let $\cD_0$ put mass $3/4$ on $s_0$ and $1/4$ on $s_1$, and let $\cD_1$ put mass $3/4$ on $s_0$ and $1/4$ on $s_2$.
Under $\cD_0$, the language $L$ has risk $0$ and every other language misses $s_1$, so $\iderr(\cA,\cD_0,\cC,n)\ge\frac14\Pr[\cA(S)\ne L]$ (Lemma~\ref{lem:id-pair}).
In the same way, $\iderr(\cA,\cD_1,\cC,n)\ge\frac14\Pr[\cA(S')=L]$ for $S'\sim\cD_1^n$.

Couple the coordinates independently: $(x_t,x'_t)=(s_0,s_0)$ with probability $3/4$, and $(s_1,s_2)$ otherwise.
Lemma~\ref{lem:lecam} with $F=\{L\}$ and $p=1/4$ gives $\Pr[\cA(S)\ne L]+\Pr[\cA(S')=L]\ge(1-\theta_\eps/4)^n$.
The last inequality uses Lemma~\ref{lem:elementary}(b) and $\theta_\eps\le\min\{1,\eps\}$.
\end{proof}

\begin{definition}[IIDP condition]\label{def:iidp}
A collection $\cC$ satisfies the \emph{intersecting infinite-difference pair} (IIDP) condition if it contains two languages $L,L'$ with $L\cap L'\ne\varnothing$ such that $L\setminus L'$ and $L'\setminus L$ are infinite.
\end{definition}

\begin{theorem}[Generation lower bound]\label{thm:gen-lb}
Let $\cC$ satisfy the IIDP condition.
There are two distributions $\cD_0,\cD_1$ such that for every $\eps\in(0,\infty]$, every $\eps$-DP generation algorithm $\cA$, and every $n\ge1$,
\[
\max_{b\in\{0,1\}}\generr(\cA,\cD_b,\cC,n)\;\ge\;\frac13\Bigl(1-\frac{\theta_\eps}4\Bigr)^n\;\ge\;\frac13\,e^{-\min\{1,\eps\}\,n/3}.
\]
\end{theorem}

\begin{proof}
Let $s_0\in L\cap L'$, and enumerate $L\setminus L'=\{a_1,a_2,\dots\}$ and $L'\setminus L=\{b_1,b_2,\dots\}$.
Let $\cD_0$ put mass $3/4$ on $s_0$ and mass $2^{-k}/4$ on $a_k$, and define $\cD_1$ in the same way with $b_k$.
Let $F=\{a_k:k\ge1\}$.
Under $\cD_0$, the string $s_0$ is in $S$ except with probability $4^{-n}$.
If it is, then every valid new output lies in $F$, so $\Pr[\cA(S)\notin F]\le\generr(\cA,\cD_0,\cC,n)+4^{-n}$.
Under $\cD_1$, no string of $F$ is in the support, so $\Pr[\cA(S')\in F]\le\generr(\cA,\cD_1,\cC,n)$ (Lemma~\ref{lem:gen-pair} checks these facts).
Couple $(x_t,x'_t)=(s_0,s_0)$ with probability $3/4$, and $(a_K,b_K)$ with $\Pr[K=k]=2^{-k}$ otherwise.
Lemma~\ref{lem:lecam} with $p=1/4$ shows that the sum of the two errors is at least $(1-\theta_\eps/4)^n-4^{-n}\ge\frac23(1-\theta_\eps/4)^n$, because $(1-\theta_\eps/4)^n\ge(3/4)^n\ge3\cdot4^{-n}$. Lemma~\ref{lem:elementary}(b) gives the last inequality of the theorem.
\end{proof}

At $\eps=\infty$, the two theorems give the non-private bounds $\frac18(3/4)^n$ and $\frac13(3/4)^n$.
So one argument covers every privacy level, and the case $\eps\ge1$ needs no separate proof.
The next example shows that the generation bound and Theorem~\ref{thm:gen-public} meet on one instance.

\begin{example}[A sharp instance]\label{ex:sharp}
Let $\cU=\N$ with $u_k=k$, $L=\{1\}\cup\{3k:k\ge1\}$, $L'=\{1\}\cup\{3k+1:k\ge1\}$, and $\cC=\{L,L'\}$.
The collection satisfies both conditions with $s_0=1$, $a_k=3k$, and $b_k=3k+1$, and $\supp(\cD_0)=L$.
So under $\cD_0$, the language $L$ is good, $L'$ is bad with witness $4$, and Assumption~\ref{asm:gen} holds with $w=4$.
With $W=4$, the target strings in $\cU_4$ are $1$ and $3$, so the mass floor is $m_4=\cD_0(3)=1/8$.
With the horizon $f=2$, Theorem~\ref{thm:gen-public} gives $\generr\le4e^{-n/64}+2e^{-\eps n/32}+n2^{-n}$ under $\cD_0$. The same computation applies under $\cD_1$, where $w=3$, the target strings in $\cU_4$ are $1$ and $4$, and again $m_4=1/8$.
Theorem~\ref{thm:gen-lb} shows that every $\eps$-DP algorithm has error at least $\frac13e^{-\min\{1,\eps\}n/3}$ under one of them.
Thus the optimal rate on this pair is $e^{-\Theta(\min\{1,\eps\}n)}$.
Remark~\ref{rem:sharp-general} in Appendix~\ref{app:lb} shows that every pair with the IIDP condition and $L\cap L'=\{s_0\}$ behaves in the same way.
\end{example}

\section{Discussion}\label{sec:discussion}

The price of privacy in language learning is set by two comparisons.
The first comparison sets the exponent. On one side is the large-deviation exponent of the sampling event that makes the right answer win. On the other side is the number of samples that an adversary must change to flip the answer.
Identification decides at a precision $1/f(n)$ that shrinks with $n$.
So the adversary's cost $n/f(n)$ exceeds the sampling exponent $n/f(n)^2$, and privacy is free.
Generation decides through counts of single strings, so both quantities are of order $n$.
Pure DP then keeps a fraction $\Theta(\theta_\eps)=\Theta(\min\{1,\eps\})$ of the exponent: at least this fraction at every distribution when a witness bound is public, and at most this fraction on the hard pairs of Section~\ref{sec:lb}. Gaussian noise, which squares the adversary's cost, keeps all of it.

The second comparison concerns what the learner can see.
A non-private learner sees the absence of a string at no cost, because one occurrence suffices. A pure-DP learner cannot separate absence from a small mass $t$ before it has about $1/t$ samples (Lemma~\ref{lem:resampling}).
Without a public witness bound, a generator must decide which strings are absent. On collections as in Theorem~\ref{thm:gen-nec}, pure DP then loses the exponential rate.
A public bound, or a known mass floor, removes this decision.

\paragraph{Limits.}
Our algorithms are information-theoretic: they enumerate languages and strings, and we do not study their running time.
We protect single strings, not users who contribute many strings.
For identification, the gap between $e^{-o(n)}$ and $e^{-O(n)}$ comes from the non-private problem~\citep{hp26}.
Our constants are not optimized.
Theorem~\ref{thm:gen-nec} covers finite collections with two infinite languages of finite intersection; we do not know exactly which collections need a public bound.
Appendix~\ref{app:limits} discusses these points and further directions.\markmainend

\section*{Acknowledgments}
This work was supported by the Australian Research Council (ARC) under the DECRA Fellowship DE240101049.
We thank the anonymous reviewers and the area chair for their helpful comments.
We used an AI assistant (Claude, Anthropic) to help check the proofs and draft parts of this paper; the authors take full responsibility for the content.
The authors declare no competing interests.

\bibliographystyle{plainnat}
\bibliography{ref}

\clearpage
\appendix
\addtocontents{toc}{\protect\addvspace{5pt}\protect\noindent{\protect\sffamily\protect\small\protect\bfseries\protect\color{muted}Appendix}\protect\par}
\noindent{\color{refcool}\rule{\linewidth}{1.6pt}}\par\vspace{8pt}
\noindent{\sffamily\bfseries\LARGE\color{navy}Appendix}\par\vspace{3pt}
\noindent{\sffamily\small\color{muted}The Price of Privacy for Language Identification and Generation}\par\vspace{6pt}
\noindent{\color{rulegray}\rule{\linewidth}{0.5pt}}\par\vspace{4pt}
\suppressfloats[t]
\section{Roadmap of the appendix}\label{app:roadmap}

\begin{center}
\small
\begin{tabular}{@{}l>{\raggedright\arraybackslash}p{0.72\linewidth}@{}}
\toprule
Part & Content \\
\midrule
Appendix~\ref{app:tools} & Standard facts: concentration, the two mechanisms, post-processing, group privacy. \\
Appendix~\ref{app:master} & Proof of the private selection lemma; resampling under approximate DP; elementary inequalities. \\
Appendix~\ref{app:id} & Proofs for identification (Section~\ref{sec:id}). \\
Appendix~\ref{app:gen} & Proofs for generation (Section~\ref{sec:gen}), including the necessity of a public bound. \\
Appendix~\ref{app:lb} & The facts about the hard pairs of Section~\ref{sec:lb}. \\
Appendix~\ref{app:nonprivate} & The non-private rules of \citet{hp26} as limits of our scores. \\
Appendix~\ref{app:limits} & Limitations, open problems, broader impacts, more related work, and a comparison with other private learning tasks (Table~\ref{tab:compare}). \\
\bottomrule
\end{tabular}
\end{center}

\IfFileExists{figs/fig_dependency.pdf}{%
\begin{figure}[h]
\centering
\includegraphics[width=\linewidth]{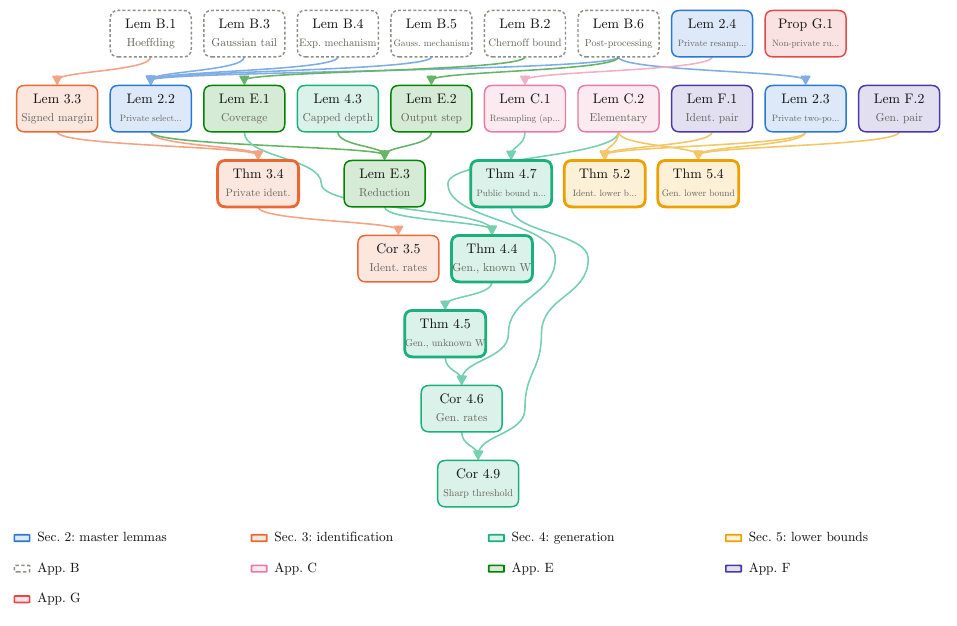}
\caption{Dependency graph of the results. An arrow from $X$ to $Y$ means that the proof of $Y$ uses $X$. Dashed frames mark standard facts that we cite.}
\label{fig:dependency}
\end{figure}}{}

\paragraph{How the proofs are organized.}
Proofs of statements with several parts treat the parts in order, labeled (a), (b), (c).
The longer arguments state their plan first.
The longest proof, of Theorem~\ref{thm:gen-nec}, is divided into paragraphs whose italic lead-ins name their jobs.
Figure~\ref{fig:dependency} shows which earlier results each proof uses.

\paragraph{Notation.}
Table~\ref{tab:notation} lists the symbols that the paper uses in more than one section.

\begin{table}[h]
\centering
\small
\caption{Notation.}
\label{tab:notation}
\begin{tabular}{@{}l>{\raggedright\arraybackslash}p{0.74\linewidth}@{}}
\toprule
Symbol & Meaning \\
\midrule
$S$, $n$, $N_S(x)$ & sample $(x_1,\dots,x_n)\sim\cD^n$, its size, and the count of $x$ in $S$ \\
$\dH(S,S')$ & Hamming distance; $S,S'$ are neighbors if $\dH(S,S')=1$ \\
$\eps,\delta$, $\theta_\eps$ & privacy parameters; $\theta_\eps=1-e^{-\eps}\asymp\min\{1,\eps\}$ \\
$\kappa_{\eps,t}$ & resampling factor $1+t(e^\eps-1)$ of Lemma~\ref{lem:resampling} \\
$q$, $\Delta$, $\Gamma$ & score, its sensitivity, and the margin of the good candidates \\
$\EM_\eps$, $\GN_{\eps,\delta}$ & exponential mechanism and Gaussian noisy argmax \\
$f=f(n)$ & horizon: the candidates are $L_1,\dots,L_f$ \\
$i^\star$ & target index (Assumptions~\ref{asm:attained} and~\ref{asm:gen}) \\
$M_S(i)$, $\cgap$, $\qid$ & empirical margin, population gap, signed-margin score (Section~\ref{sec:id}) \\
$w$, $W$ & witness index $w(\cC,\cD)$ and a public bound $W\ge w$ \\
$a_{i,\ell}(S)$, $\lambda$, $\qgen$ & prefix count, cap profile, capped-depth score (Section~\ref{sec:gen}) \\
$g=g(n)$ & threshold of the profile $\lambda_\ell=g\,\ell/(\ell+1)$ \\
$\cU_\ell$, $m_\ell$ & the first $\ell$ strings; the mass floor $\min\{\cD(u):u\in L_{i^\star}\cap\cU_\ell\}$ \\
$\underline m$ & a known lower bound on the mass floor $m_w$ \\
\bottomrule
\end{tabular}
\end{table}

\section{Auxiliary results}\label{app:tools}

We use the following standard facts without proof. Lemmas~\ref{lem:hoeffding}--\ref{lem:gauss-tail} are in \citet{blm13}; Lemma~\ref{lem:em} is due to \citet{mt07}, Lemma~\ref{lem:gauss-mech} to \citet{dkm+06}, and \citet{dr14} prove Lemmas~\ref{lem:em}--\ref{lem:postproc}.

\begin{lemma}[Hoeffding's inequality]\label{lem:hoeffding}
Let $X$ be the average of $n$ independent random variables with values in $[0,1]$, and let $\mu=\E[X]$.
For every $t\ge0$, $\Pr[|X-\mu|\ge t]\le2e^{-2nt^2}$.
\end{lemma}

\begin{lemma}[Chernoff bound]\label{lem:chernoff}
Let $N$ have the binomial law with parameters $n$ and $p$.
Then $\Pr[N\le np/2]\le e^{-np/8}$.
\end{lemma}

\begin{lemma}[Gaussian tail]\label{lem:gauss-tail}
If $Z\sim\cN(0,\sigma^2)$ and $t\ge0$, then $\Pr[|Z|\ge t]\le2e^{-t^2/(2\sigma^2)}$.
\end{lemma}

\begin{lemma}[Exponential mechanism]\label{lem:em}
For every finite $\cR$, every $\Delta>0$, and every score $q$ with sensitivity at most $\Delta$, the mechanism $\EM_\eps(\cdot,q)$ with sensitivity bound $\Delta$ is $\eps$-DP.
\end{lemma}

\begin{lemma}[Gaussian mechanism]\label{lem:gauss-mech}
Let $v:\cU^n\to\R^d$ change by at most $\Delta_2$ in $\ell_2$-norm between neighbors, and let $\eps,\delta\in(0,1)$.
Then $S\mapsto v(S)+Z$ with $Z\sim\cN(0,\sigma^2I_d)$ and $\sigma=\Delta_2\sqrt{2\log(1.25/\delta)}/\eps$ is $(\eps,\delta)$-DP.
\end{lemma}

\begin{lemma}[Post-processing and group privacy]\label{lem:postproc}
If $\cA$ is $(\eps,\delta)$-DP and $h$ is a randomized map whose randomness is independent of the data, then $h\circ\cA$ is $(\eps,\delta)$-DP.
If $\cA$ is $\eps$-DP, then $\Pr[\cA(S)\in F]\le e^{\eps\dH(S,S')}\Pr[\cA(S')\in F]$ for all $S,S'\in\cU^n$ and every event $F$.
\end{lemma}

\section{Proofs for Section~\ref{sec:setup}}\label{app:master}

\begin{proof}[Proof of Lemma~\ref{lem:selection}]
Part (a) collects the privacy guarantees of the two mechanisms.
The exponential mechanism is $\eps$-DP by Lemma~\ref{lem:em}.
For the Gaussian mechanism, consider the score vector $v(S)=(q(S,c))_{c\in\cR}$.
Each of its $|\cR|$ coordinates changes by at most $\Delta$ between neighbors, so the vector changes by at most $\Delta\sqrt{|\cR|}$ in $\ell_2$-norm.
For $\eps,\delta\in(0,1)$ and the stated value of $\sigma$, Lemma~\ref{lem:gauss-mech} shows that the release of $v(S)+Z$ is $(\eps,\delta)$-DP.
The selected candidate is a function of $v(S)+Z$, so it is $(\eps,\delta)$-DP by Lemma~\ref{lem:postproc}.

For part (b), let $c^\star$ be a good candidate with the largest score.
By the definition of the margin, every bad candidate $c$ has $q(S,c)\le q(S,c^\star)-\Gamma$.
The normalizing sum of the exponential mechanism is at least the weight of $c^\star$, so the mechanism selects a bad candidate $c$ with probability at most
\[
\exp\Bigl(\frac{\eps\,(q(S,c)-q(S,c^\star))}{2\Delta}\Bigr)\le e^{-\eps\Gamma/(2\Delta)}.
\]
A union bound over the at most $|\cR|$ bad candidates gives the first inequality.
For the Gaussian mechanism, suppose that $|Z_c|<\Gamma/2$ for every $c\in\cR$.
Then every bad candidate $c$ satisfies
\[
q(S,c^\star)+Z_{c^\star}>q(S,c^\star)-\Gamma/2\ge q(S,c)+\Gamma/2>q(S,c)+Z_c,
\]
so no bad candidate maximizes the noisy score.
By Lemma~\ref{lem:gauss-tail} with $t=\Gamma/2$ and a union bound over $\cR$, the event that $|Z_c|\ge\Gamma/2$ for some $c$ has probability at most $2|\cR|\,e^{-\Gamma^2/(8\sigma^2)}$.
This gives the second inequality.
\end{proof}

The proofs of Lemmas~\ref{lem:lecam} and~\ref{lem:resampling} are complete in Section~\ref{sec:setup}.
The next lemma is the approximate-DP form of Lemma~\ref{lem:resampling}.

\begin{lemma}[Resampling under approximate DP]\label{lem:resampling-approx}
Let $P,R_0,R_1$, $t$, and $\cD_0,\cD_1$ be as in Lemma~\ref{lem:resampling}, and let $\cA$ be $(\eps,\delta)$-DP.
Then for every set $F$ of outputs, $\Pr_{\cD_1^n}[\cA\in F]\le\kappa_{\eps,t}^{\,n}\Pr_{\cD_0^n}[\cA\in F]+nt\,\kappa_{\eps,t}^{\,n-1}\delta$.
\end{lemma}

\begin{proof}
The proof is the hybrid argument of Lemma~\ref{lem:resampling}, with one additive term per step.
As there, let $Q_j$ be the law of a sample whose first $j$ entries come from $\cD_1$ and whose other entries come from $\cD_0$.
Fix $j\in[n]$ and every entry except the $j$-th, and let $\psi(x)$ be the probability, over the coins of $\cA$, that $\cA$ outputs a point of $F$ when the $j$-th entry is $x$.
Two samples that differ only in the $j$-th entry are neighbors, so $(\eps,\delta)$-DP gives $\psi(x)\le e^\eps\psi(x')+\delta$ for all strings $x,x'$.
In other words, $\sup\psi\le e^\eps\inf\psi+\delta$.
Because $\cD_1-\cD_0=t(R_1-R_0)$, it follows that
\[
\E_{x\sim\cD_1}\psi-\E_{x\sim\cD_0}\psi\le t(\sup\psi-\inf\psi)\le t(e^\eps-1)\,\E_{x\sim\cD_0}\psi+t\delta.
\]
Averaging over the other entries gives $Q_j(F)\le\kappa_{\eps,t}\,Q_{j-1}(F)+t\delta$ for every $j\in[n]$.
Unroll this recursion from $Q_0=\cD_0^n$ to $Q_n=\cD_1^n$: the result is $Q_n(F)\le\kappa_{\eps,t}^{\,n}Q_0(F)+t\delta\sum_{k=0}^{n-1}\kappa_{\eps,t}^{\,k}$.
The sum is at most $n\kappa_{\eps,t}^{\,n-1}$, because $\kappa_{\eps,t}\ge1$.
\end{proof}

The next lemma collects the elementary inequalities that the paper uses.

\begin{lemma}[Elementary inequalities]\label{lem:elementary}
\begin{enumerate}[label=(\alph*)]
\item For every $\eps\in(0,\infty]$, $(1-e^{-1})\min\{1,\eps\}\le1-e^{-\eps}\le\min\{1,\eps\}$.
\item For every $x\in[0,1/4]$, $1-x\ge e^{-4x/3}$.
\item For every $n\ge14$, $n2^{-n}\le e^{-n/2}$.
\end{enumerate}
\end{lemma}

\begin{proof}
(a) The upper bound holds because $1-e^{-\eps}\le1$ and $e^{-\eps}\ge1-\eps$.
For the lower bound, if $\eps\ge1$, then $1-e^{-\eps}\ge1-e^{-1}$.
If $\eps<1$, then concavity gives the bound: the map $\eps\mapsto1-e^{-\eps}$ is concave and vanishes at $0$, so on $[0,1]$ it lies above its chord, and $1-e^{-\eps}\ge\eps(1-e^{-1})$.

(b) Let $x\in[0,1/4]$, so that $1-x\ge3/4$. Expand the logarithm, compare it with a geometric series, and use $1-x\ge3/4$ in the last step:
\[
\log(1-x)=-\sum_{k\ge1}\frac{x^k}{k}\ge-\sum_{k\ge1}x^k=-\frac{x}{1-x}\ge-\frac{4x}{3}.
\]

(c) Taking logarithms, the claim is equivalent to $\log n\le(\log2-\tfrac12)\,n$.
It holds at $n=14$, because $\log14<2.64<0.193\cdot14$.
Beyond $n=14$, the right side grows by more than $0.19$ per unit step, and the left side grows by less than $1/14<0.08$.
\end{proof}

\section{Proofs for Section~\ref{sec:id}}\label{app:id}

Throughout, $f=f(n)$, and $\cE$ is the event of Lemma~\ref{lem:qid}(b).
On $\cE$, every empirical risk $\err_S(L_i)$ with $i\in[f]$ is within $1/(4f)$ of its mean, so every difference of two such empirical risks is within $1/(2f)$ of the corresponding difference of risks.

\begin{proof}[Proof of Lemma~\ref{lem:qid}]
Parts (a) and (b) are routine. Part (c) carries the weight: it compares the score of the target index with the scores of the earlier and the later indices.

(a) Neighboring samples differ in one entry, and this entry changes at most one indicator in the average that defines an empirical risk.
So every empirical risk changes by at most $1/n$, and every difference $\err_S(L_j)-\err_S(L_i)$ changes by at most $2/n$.
A minimum of functions that each change by at most $2/n$ changes by at most $2/n$, so the same bound holds for $M_S(i)$; and $M_S(1)=1$ does not change at all.
The score $\qid(S,i)$ is a minimum of terms of the form $\pm M_S(k)$ plus constants, so it also changes by at most $2/n$.

(b) For each $i$, the empirical risk $\err_S(L_i)$ is an average of $n$ independent indicators with mean $\err_\cD(L_i)$.
Lemma~\ref{lem:hoeffding} with $t=1/(4f)$ bounds the probability of a deviation by $2e^{-n/(8f^2)}$, and a union bound over the $f$ indices gives part (b).

(c) Assume that $f\ge\max\{i^\star,3/\cgap,3\}$ and that $\cE$ holds.
The proof rests on two estimates of the empirical margins, one for the target index and one for the later indices.

First, $M_S(i^\star)\ge5/(2f)$.
If $i^\star=1$, then this holds because $M_S(1)=1$ and $f\ge3$.
If $i^\star\ge2$, then every $j<i^\star$ satisfies $\err_S(L_j)-\err_S(L_{i^\star})\ge\cgap-1/(2f)\ge5/(2f)$, by the observation at the start of this appendix and because $\cgap\ge3/f$.

Second, every $k$ with $i^\star<k\le f$ has $M_S(k)\le1/(2f)$.
Indeed, the index $j=i^\star$ is a candidate in the minimum that defines $M_S(k)$, and the same observation and the optimality of $L_{i^\star}$ give
\[
\err_S(L_{i^\star})-\err_S(L_k)\le\err_\cD(L_{i^\star})-\err_\cD(L_k)+\frac1{2f}\le\frac1{2f}.
\]

Now compare the scores.
The index $i^\star$ lies in $[f]$, because $f\ge i^\star$.
By the first estimate, the first term of $\qid(S,i^\star)$ is at least $5/(2f)-2/f=1/(2f)$, and by the second, every later term is at least $2/f-1/(2f)=3/(2f)$; so $\qid(S,i^\star)\ge1/(2f)$.
An earlier index $i<i^\star$ has $i^\star$ in its inner minimum, so $\qid(S,i)\le2/f-M_S(i^\star)\le-1/(2f)$.
A later index $i$ with $i^\star<i\le f$ has $\qid(S,i)\le M_S(i)-2/f\le-3/(2f)$, by its first term and the second estimate.
The margin is therefore at least $1/(2f)+1/(2f)=1/f$.
\end{proof}

\begin{proof}[Proof of Theorem~\ref{thm:id}]
Both identifiers run their mechanism with $\Delta=2/n$, which bounds the sensitivity of $\qid$ by Lemma~\ref{lem:qid}(a).
Lemma~\ref{lem:selection}(a) gives the privacy of the selection, and the output $L_{\wh i}$ is a function of $\wh i$.

For the error, note that the excess risk of $L_{\wh i}$ lies in $[0,1]$ and vanishes if $\wh i=i^\star$.
Conditioning on $S$ therefore gives
\[
\iderr\le\Pr[\wh i\ne i^\star]\le\Pr[\cE^c]+\sup_{S\in\cE}\Pr[\wh i\ne i^\star\mid S].
\]
The first term is at most $2f\,e^{-n/(8f^2)}$ by Lemma~\ref{lem:qid}(b).
For the second, fix $S\in\cE$.
Because $f\ge\max\{i^\star,3/\cgap,3\}$, Lemma~\ref{lem:qid}(c) gives a margin of at least $\Gamma=1/f$ for the set $\cG=\{i^\star\}$, and Lemma~\ref{lem:selection}(b) with $\Delta=2/n$ and $|\cR|=f$ bounds the probability of a wrong selection.
For the exponential mechanism, $\eps\Gamma/(2\Delta)=\eps n/(4f)$, and the bound is $f\,e^{-\eps n/(4f)}$.
For the Gaussian mechanism, $\sigma^2=8f\log(1.25/\delta)/(\eps^2n^2)$, so $\Gamma^2/(8\sigma^2)=\eps^2n^2/(64f^3\log(1.25/\delta))$, and the bound is the second term of part (b).
\end{proof}

\begin{proof}[Proof of Corollary~\ref{cor:id-rates}]
By hypothesis, $f\ge\max\{i^\star,3/\cgap,3\}$, so Theorem~\ref{thm:id} applies; because $r(n)=o(n)$, the horizon $f(n)$ tends to infinity, so this hypothesis holds for all large $n$.
We may assume $r\ge\log4$, because otherwise $4e^{-r}>1$ and there is nothing to prove; then $f\le\sqrt{n/(16r)}\le\sqrt n$.
The inequality $f\le\sqrt{n/(16r)}$ gives $8f^2\le n/(2r)$, that is, $n/(8f^2)\ge2r$.
So the first term of the theorem is at most $2f\,e^{-2r}\le2\sqrt n\,e^{-2r}\le2e^{-r}$, where the last step uses $\sqrt n\le e^{r}$, a consequence of $r\ge\log n$.

Next, we bound the privacy terms.
If $\eps\ge2\sqrt{r/n}$, then the exponent of part (a) satisfies $\eps n/(4f)\ge\eps\sqrt{nr}\ge2r$, so the privacy term is at most $f\,e^{-2r}\le e^{-r}$.
If $\eps^2\sqrt{nr}\ge2\log(1.25/\delta)$, then $f^3\le n^{3/2}/(64r^{3/2})$ shows that the exponent of part (b) is at least $\eps^2\sqrt n\,r^{3/2}/\log(1.25/\delta)\ge2r$, so the privacy term is at most $2f\,e^{-2r}\le2e^{-r}$.
Adding the terms, the error is at most $3e^{-r}$ for the pure identifier and at most $4e^{-r}$ for the Gaussian identifier.
\end{proof}

\section{Proofs for Section~\ref{sec:gen}}\label{app:gen}

All samples below have their entries in $\supp(\cD)$, which holds with probability one; so Lemma~\ref{lem:qgen} applies to them.
We write $\cG\subseteq[f]$ for the set of good languages in the horizon.

\begin{lemma}[Coverage]\label{lem:coverage}
Let $\ell\in\N$ and $g\le nm_\ell/2$.
The event that $N_S(u)\ge g$ for every $u\in L_{i^\star}\cap\cU_\ell$ fails with probability at most $\ell\,e^{-nm_\ell/8}$.
\end{lemma}

\begin{proof}
For $u\in L_{i^\star}\cap\cU_\ell$, the count $N_S(u)$ has the binomial law with parameters $n$ and $\cD(u)\ge m_\ell$.
Because $g\le n\cD(u)/2$, Lemma~\ref{lem:chernoff} gives $\Pr[N_S(u)<g]\le e^{-n\cD(u)/8}\le e^{-nm_\ell/8}$.
A union bound over the at most $\ell$ strings of $L_{i^\star}\cap\cU_\ell$ finishes the proof.
\end{proof}

\begin{lemma}[Output step]\label{lem:output}
The generator has the privacy of its selection step.
If the selected language $L_{\wh i}$ is good, the output lies in $\supp(\cD)\setminus S$ except with probability at most $n2^{-n}$.
\end{lemma}

\begin{proof}
The output is a function of $\wh i$ and of an index $j$ that is independent of $S$, so Lemma~\ref{lem:postproc} applies.
Fix $S$ and a good $\wh i$.
The language $L_{\wh i}$ is infinite, so its first $2^n$ strings exist, and all of them are valid.
At most $n$ of them lie in $S$, and $j$ is uniform on $[2^n]$.
\end{proof}

\begin{lemma}[Reduction to a margin]\label{lem:gen-reduction}
Let $\lambda$ be a cap profile, let $\cE^\star$ be an event, and assume that on $\cE^\star$ the margin of $\cG$ for the score $\qgen$ is at least $\Gamma>0$.
Then the generator that uses $\EM_\eps$ has $\generr\le\Pr[(\cE^\star)^c]+f\,e^{-\eps\Gamma/2}+n2^{-n}$, and the generator that uses $\GN_{\eps,\delta}$ has $\generr\le\Pr[(\cE^\star)^c]+2f\,e^{-\eps^2\Gamma^2/(16f\log(1.25/\delta))}+n2^{-n}$.
\end{lemma}

\begin{proof}
A failure of the generator requires one of three events: the event $\cE^\star$ fails; it holds and the selected index is bad ($\wh i\notin\cG$); or $\wh i\in\cG$ and the output lies in $S$.
Indeed, if $\wh i\in\cG$ and the output is not in $S$, then the output is valid and new by Lemma~\ref{lem:output}.
The first event has probability $\Pr[(\cE^\star)^c]$, and the third has probability at most $n2^{-n}$ by Lemma~\ref{lem:output}.
For the second, apply Lemma~\ref{lem:selection}(b) on $\cE^\star$ with $\Delta=1$ (Lemma~\ref{lem:qgen}(a)) and $|\cR|=f$.
It bounds $\Pr[\wh i\notin\cG\mid S]$ by $f\,e^{-\eps\Gamma/2}$ for $\EM_\eps$.
For $\GN_{\eps,\delta}$, the noise variance is $\sigma^2=2f\log(1.25/\delta)/\eps^2$, so $\Gamma^2/(8\sigma^2)=\eps^2\Gamma^2/(16f\log(1.25/\delta))$, and the bound is $2f\,e^{-\eps^2\Gamma^2/(16f\log(1.25/\delta))}$.
A union bound over the three events finishes the proof.
\end{proof}

\begin{proof}[Proof of Theorem~\ref{thm:gen-public}]
The generators have the stated privacy by Lemma~\ref{lem:qgen}(a), Lemma~\ref{lem:selection}(a), and Lemma~\ref{lem:output}.
For the error, we exhibit a likely event on which the good languages win by a margin of $nm_W/2$, and then apply Lemma~\ref{lem:gen-reduction}.

For the step profile, the score is the plain prefix count: $\qgen(S,i)=a_{i,W}(S)$ for every $i$.
Indeed, for $\ell<W$ the term $\min\{a_{i,\ell}(S),0\}$ is $0$.
For $\ell\ge W$ the term is $a_{i,\ell}(S)\le a_{i,W}(S)$, because the prefix count does not increase with $\ell$; and the term at $\ell=W$ equals $a_{i,W}(S)\ge0$.

Let $\cE^\star$ be the event that $N_S(u)\ge nm_W/2$ for every $u\in L_{i^\star}\cap\cU_W$.
By Lemma~\ref{lem:coverage} with $\ell=W$ and $g=nm_W/2$, it fails with probability at most $W e^{-nm_W/8}$.
The target index $i^\star$ lies in $\cG$, because $f\ge i^\star$.
On $\cE^\star$, its score $a_{i^\star,W}(S)$ is at least $nm_W/2$; in the trivial case $L_{i^\star}\cap\cU_W=\varnothing$, it equals $n$.
Every bad language has $\qgen(S,i)\le\lambda_{w-1}=0$ by Lemma~\ref{lem:qgen}(b), because $w-1<W$.
So the margin of $\cG$ on $\cE^\star$ is at least $nm_W/2$, and Lemma~\ref{lem:gen-reduction} with $\Gamma=nm_W/2$ gives both bounds.
\end{proof}

\begin{proof}[Proof of Theorem~\ref{thm:gen-nopublic}]
The privacy follows as in the proof of Theorem~\ref{thm:gen-public}.
For the error, the argument is the same as there, but now the margin comes from the profile at the unknown level $w$.

Let $\cE^\star$ be the event that $N_S(u)\ge g$ for every $u\in L_{i^\star}\cap\cU_w$.
Because $g\le nm_w/2$, Lemma~\ref{lem:coverage} with $\ell=w$ shows that $\cE^\star$ fails with probability at most $w e^{-nm_w/8}$.
The target index lies in $\cG$, because $f\ge i^\star$.
On $\cE^\star$, its prefix count $a_{i^\star,w}(S)$ is at least $g$; in the trivial case $L_{i^\star}\cap\cU_w=\varnothing$, it equals $n\ge g$.
Because $\lambda_w<g$, Lemma~\ref{lem:qgen}(c) with $\ell=w$ gives $\qgen(S,i^\star)\ge\lambda_w$.
Every bad language has $\qgen(S,i)\le\lambda_{w-1}$ by Lemma~\ref{lem:qgen}(b).
The caps at the levels $w$ and $w-1$ differ by
\[
\lambda_w-\lambda_{w-1}=g\Bigl(\frac{w}{w+1}-\frac{w-1}{w}\Bigr)=\frac{g}{w(w+1)},
\]
so the margin of $\cG$ on $\cE^\star$ is at least $g/(w(w+1))$, and Lemma~\ref{lem:gen-reduction} with $\Gamma=g/(w(w+1))$ gives both bounds.
\end{proof}

\begin{proof}[Proof of Corollary~\ref{cor:gen-rates}]
All three parts apply Theorem~\ref{thm:gen-nopublic} with a suitable threshold $g$, together with the bound $n2^{-n}\le e^{-n/2}$ for $n\ge14$ (Lemma~\ref{lem:elementary}(c)).
As in Section~\ref{sec:gen}, the horizon $f(n)$ tends to infinity, so $f(n)\ge i^\star$ for large $n$, and the theorem applies.

(a) The choice $g=\lfloor n\underline m/2\rfloor$ satisfies $g\le nm_w/2$, because $\underline m\le m_w$; and $g\ge n\underline m/4$ for large $n$.
Theorem~\ref{thm:gen-nopublic}(a) then gives $\generr\le(w+f+1)\exp(-\min\{1,\eps/(w(w+1))\}\,n\underline m/8)$.
Because $\log f(n)=o(\min\{1,\eps\}\,n)$, the factor $w+f+1$ does not change the rate $e^{-\Omega(\min\{1,\eps\}n)}$.

(b) Because $\psi(n)\to\infty$ and $\psi(n)=o(n)$, for large $n$ we have $n/(2\psi)\le g\le nm_w/2$ and $nm_w/8\ge n/\psi$.
Theorem~\ref{thm:gen-nopublic}(a) then gives $\generr\le(w+f+1)\exp(-\min\{1,\eps\}\,n/(4w(w+1)\psi))$, and the assumption $\log f(n)=o(\min\{1,\eps\}\,n/\psi(n))$ makes the factor $w+f+1$ harmless.

(c) As in part (b), $n/(2\psi)\le g\le nm_w/2$ for large $n$.
By the assumption $\eps^2n\ge\psi^3f\log(1.25/\delta)$, the exponent of the privacy term of Theorem~\ref{thm:gen-nopublic}(b) is at least
\[
\frac{\eps^2n^2}{64w^2(w+1)^2\psi^2f\log(1.25/\delta)}\ge\frac{n\psi}{64w^2(w+1)^2}.
\]
This exceeds $n$ for large $n$, so $\generr\le w\,e^{-nm_w/8}+2f\,e^{-n}+e^{-n/2}$, which is $e^{-\Omega(n)}$ because $\log f(n)=o(n)$.
\end{proof}

\begin{proof}[Proof of Theorem~\ref{thm:gen-nec}]
Let $\ell_0$ be the first level with $L\cap\cU_{\ell_0}\ne\varnothing$.
For $\ell\ge\ell_0$, let $t_\ell=\cD(L\setminus\cU_\ell)$ be the mass of $L$ beyond level $\ell$, and let $P_\ell=\cD(\cdot\mid L\cap\cU_\ell)$ and $R_\ell=\cD(\cdot\mid L\setminus\cU_\ell)$.
Fix a distribution $R'$ with support $L'$, and put $\cD'_\ell=(1-t_\ell)P_\ell+t_\ell R'$ and $F=L'\setminus L$.
The plan is to show that, under $\cD'_\ell$, the generator outputs a point of $F$ with probability close to $1$, and then to transfer this probability to $\cD$ with Lemma~\ref{lem:resampling}; under $\cD$, every point of $F$ is invalid.

\emph{The alternatives.}
The support $L$ of $\cD$ is infinite, so for $\ell\ge\ell_0$ both $L\cap\cU_\ell$ and $L\setminus\cU_\ell$ have positive mass.
So $t_\ell\in(0,1)$ and $\cD=(1-t_\ell)P_\ell+t_\ell R_\ell$.
The masses $t_\ell$ are the tails of a convergent series, so $t_\ell\to0$.
Both $\cD$ and every $\cD'_\ell$ satisfy Assumption~\ref{asm:gen}: the collection is finite, so every witness index is finite; the language $L$ is good under $\cD$; and the language $L'$ is good under $\cD'_\ell$, because $\supp(\cD'_\ell)=(L\cap\cU_\ell)\cup L'$ includes $L'$.

\emph{The set $F$ under $\cD$.}
The set $F$ does not meet $\supp(\cD)=L$, so every output in $F$ is invalid under $\cD$, and
$\Pr_{\cD^n}[\cA\in F]\le\generr(\cA,\cD,\cC,n)$.

\emph{The set $F$ under $\cD'_\ell$.}
Fix $\ell\ge\ell_0$; we show that $\Pr_{\cD_\ell'^n}[\cA\in F]\to1$ as $n\to\infty$.
Let $T_\ell=(L\cap\cU_\ell)\cup(L\cap L')$.
This set is finite, because $L\cap L'$ is finite.
Each of its strings has positive $\cD'_\ell$-mass: a string $x$ of $L\cap\cU_\ell$ has mass at least $(1-t_\ell)P_\ell(x)>0$, and a string $x$ of $L\cap L'$ has mass at least $t_\ell R'(x)>0$.
Moreover, the set $F\cup T_\ell$ includes $\supp(\cD'_\ell)=(L\cap\cU_\ell)\cup(L'\setminus L)\cup(L'\cap L)$.
So if $S$ contains every string of $T_\ell$, then every valid new output lies in $F$.
Let $\nu_\ell$ be the smallest $\cD'_\ell$-mass on $T_\ell$; by a union bound, $S$ misses a string of $T_\ell$ with probability at most $|T_\ell|(1-\nu_\ell)^n$.
Consequently,
\[
\Pr_{\cD_\ell'^n}[\cA\in F]\ge1-\generr(\cA,\cD'_\ell,\cC,n)-|T_\ell|(1-\nu_\ell)^n.
\]
Both subtracted terms tend to $0$: the first because the error of $\cA$ tends to $0$ at every distribution that satisfies Assumption~\ref{asm:gen}, and $\cD'_\ell$ is one of them; the second because $\nu_\ell>0$.

\emph{The transfer.}
The decomposition $\cD=(1-t_\ell)P_\ell+t_\ell R_\ell$ shows that the pair $\cD_0=\cD$ and $\cD_1=\cD'_\ell$ has the form of Lemma~\ref{lem:resampling}, with $P=P_\ell$, $R_0=R_\ell$, $R_1=R'$, and $t=t_\ell$.
If $\cA$ is $\eps$-DP, then the lemma and the bound $\Pr_{\cD^n}[\cA\in F]\le\generr(\cA,\cD,\cC,n)$ give, for every $\ell\ge\ell_0$ and every $n$,
\begin{equation}\label{eq:transfer-pure}
\generr(\cA,\cD,\cC,n)\ge\kappa_{\eps,t_\ell}^{-n}\Pr_{\cD_\ell'^n}[\cA\in F].
\end{equation}
If $\cA$ is $(\eps,\delta_n)$-DP, then Lemma~\ref{lem:resampling-approx} gives in the same way, because $\kappa_{\eps,t_\ell}\ge1$,
\begin{equation}\label{eq:transfer-approx}
\generr(\cA,\cD,\cC,n)\ge\kappa_{\eps,t_\ell}^{-n}\Pr_{\cD_\ell'^n}[\cA\in F]-nt_\ell\delta_n.
\end{equation}

\emph{The conclusion.}
Fix $c>0$, and set $c_0=\infty$ under pure DP.
Lemma~\ref{lem:resampling} gives $\log\kappa_{\eps,t}\le t(e^\eps-1)$, and $t_\ell\to0$, so there is a level $\ell\ge\ell_0$ with $\log\kappa_{\eps,t_\ell}<\min\{c,c_0\}/2$.
For this $\ell$, the limit $\Pr_{\cD_\ell'^n}[\cA\in F]\to1$ gives $\Pr_{\cD_\ell'^n}[\cA\in F]\ge3/4$ for all large $n$.
Under pure DP, \eqref{eq:transfer-pure} then gives $\generr\ge\frac34e^{-cn/2}\ge\frac12e^{-cn}$.
Under approximate DP, \eqref{eq:transfer-approx} gives $\generr\ge\frac34e^{-\min\{c,c_0\}n/2}-ne^{-c_0n}$.
The second term is of smaller order, because $c_0>\min\{c,c_0\}/2$; so $\generr\ge\frac12e^{-cn}$ for all large $n$.
\end{proof}

\begin{proof}[Proof of Corollary~\ref{cor:gen-threshold}]
Part (a) combines the upper bound of Corollary~\ref{cor:gen-rates}(b) with Theorem~\ref{thm:gen-nec}. Part (b) proves its two directions separately: Corollary~\ref{cor:gen-rates}(c) gives the first, and the transfer inequality~\eqref{eq:transfer-approx} gives the second.

(a) The collection is finite, so Assumption~\ref{asm:gen} holds for each distribution in question.
Given $r(n)=o(n)$, we may assume $r(n)\to\infty$, because replacing $r$ by $\max\{r,\log n\}$ only makes the rate faster. Put $\tilde r(n)=\sqrt{n\,r(n)}$, so that $\tilde r(n)=o(n)$ and $\tilde r(n)/r(n)\to\infty$.
Corollary~\ref{cor:gen-rates}(b) with $\tilde r$ in place of $r$, and with a horizon $f(n)\to\infty$ that grows so slowly that $\log f(n)=o(\tilde r(n))$, gives the rate $e^{-\Omega(\min\{1,\eps\}\tilde r(n))}$, which is at most $e^{-r(n)}$ for large $n$.
The generator of that corollary is consistent, so Theorem~\ref{thm:gen-nec} applies to it, as it does to every other consistent $\eps$-DP generator; none of them has an exponential rate.
The restriction to consistent generators costs nothing, because a generator with an exponential rate at every distribution that satisfies Assumption~\ref{asm:gen} is consistent.

(b) Suppose first that $\log(1/\delta_n)=o(n)$.
We may assume $\eps<1$, because a smaller $\eps$ only strengthens the privacy.
Let $\psi(n)=f(n)=\lfloor(\eps^2n/\log(1.25/\delta_n))^{1/4}\rfloor$.
Then $\psi$ and $f$ tend to infinity, $\psi(n)=o(n)$, $\log f(n)=o(n)$, and $\psi^3f\log(1.25/\delta_n)\le\eps^2n$.
So Corollary~\ref{cor:gen-rates}(c) gives an $(\eps,\delta_n)$-DP generator with an exponential rate at every distribution that satisfies Assumption~\ref{asm:gen}.

Suppose next that $\log(1/\delta_n)$ is not $o(n)$.
Then there are $c_0>0$ and infinitely many $n$ with $\delta_n\le e^{-c_0n}$.
Fix $c>0$, and choose the level $\ell$ as in the paragraph headed \emph{The conclusion} in the proof of Theorem~\ref{thm:gen-nec}.
Along these $n$, inequality~\eqref{eq:transfer-approx} gives $\generr\ge\frac34e^{-\min\{c,c_0\}n/2}-ne^{-c_0n}\ge\frac12e^{-cn}$ for all large such $n$.
So $\generr\ge\frac12e^{-cn}$ for infinitely many $n$, and no consistent $(\eps,\delta_n)$-DP generator has an exponential rate at a distribution $\cD$ with $\supp(\cD)=L$.
\end{proof}

\begin{proof}[Proof of the claim in Remark~\ref{rem:realizable}]
Let $\cC_{\mathrm{r}}$ consist of $L$ and the languages $L'_\ell=(L\cap\cU_\ell)\cup L'$ for $\ell\in\N$, and let $\cD$ and $\cD'_\ell$ be as in the proof of Theorem~\ref{thm:gen-nec}.
These distributions are realizable: $\supp(\cD)=L$ and $\supp(\cD'_\ell)=(L\cap\cU_\ell)\cup L'=L'_\ell$.
The proof of Theorem~\ref{thm:gen-nec} uses the collection only through Assumption~\ref{asm:gen} for $\cD$ and for every $\cD'_\ell$, and through the consistency of $\cA$ on these distributions.
It therefore suffices to check Assumption~\ref{asm:gen}.

Under $\cD$, the language $L$ is good.
Every other language of $\cC_{\mathrm{r}}$ includes $L'$, so it contains the first string of $L'\setminus L$, which lies outside $\supp(\cD)$.
This string is a common witness at a fixed level.

Under $\cD'_\ell$, the language $L'_\ell$ is good.
Let $v$ be the level of the first string of $L$ beyond level $\ell$ that is not in $L'$; it exists because $L\setminus L'$ is infinite.
This string lies outside $\supp(\cD'_\ell)$, and it belongs to $L$ and to every $L'_m$ with $m\ge v$; so these languages have a witness at level $v$.
The other languages are good.
For $\ell\le m<v$, every string of $L$ between the levels $\ell$ and $m$ lies in $L'$, so $L'_m=L'_\ell$.
For $m<\ell$, the language $L'_\ell$ includes $L'_m$.
So the witness index of $\cD'_\ell$ is at most $v$.
\end{proof}

\section{Details for Section~\ref{sec:lb}}\label{app:lb}

The proofs of Theorems~\ref{thm:id-lb} and~\ref{thm:gen-lb} in Section~\ref{sec:lb} are complete up to the facts below.

\begin{lemma}[The identification pair]\label{lem:id-pair}
Let $s_0,s_1,s_2$ and $\cD_0,\cD_1$ be as in the proof of Theorem~\ref{thm:id-lb}.
\begin{enumerate}[label=(\alph*)]
\item The strings $s_0,s_1,s_2$ are distinct.
\item Under $\cD_0$, the language $L$ has risk $0$, and every other language of $\cC$ has risk at least $1/4$. The same holds for $\cD_1$ and $L'$.
\item Both distributions satisfy Assumption~\ref{asm:attained}, and $\iderr(\cA,\cD_0,\cC,n)\ge\frac14\Pr[\cA(S)\ne L]$ and $\iderr(\cA,\cD_1,\cC,n)\ge\frac14\Pr[\cA(S')=L]$.
\end{enumerate}
\end{lemma}

\begin{proof}
(a) The string $s_1$ is private to $L$, so it is not in $L'$; but $s_0$ and $s_2$ are in $L'$.
In the same way, $s_2$ is not in $L$, but $s_0$ is.
So the three strings are distinct.

(b) The support of $\cD_0$ is $\{s_0,s_1\}$, and $L$ includes it, so $L$ has risk $0$.
Every other language misses the private string $s_1$, which has mass $1/4$.
The argument for $\cD_1$ is the same, with $s_2$ in place of $s_1$ and $L'$ in place of $L$.

(c) By part (b), the language $L$ attains the infimum $0$ under $\cD_0$, and likewise $L'$ attains it under $\cD_1$; so both distributions satisfy Assumption~\ref{asm:attained}.
Under $\cD_0$, every output other than $L$ has excess risk at least $1/4$.
Under $\cD_1$, the output $L$ is not $L'$, so it has excess risk at least $1/4$.
These two facts give the two inequalities.
\end{proof}

\begin{lemma}[The generation pair]\label{lem:gen-pair}
Let $s_0$, $(a_k)$, $(b_k)$, $\cD_0$, $\cD_1$, and $F$ be as in the proof of Theorem~\ref{thm:gen-lb}.
\begin{enumerate}[label=(\alph*)]
\item The set $F$ does not meet $\supp(\cD_1)$.
\item Under $\cD_0$, every valid new output outside $F$ requires $s_0\notin S$, and $\Pr[s_0\notin S]=4^{-n}$.
\item The coupling has marginals $\cD_0$ and $\cD_1$, and $x_t\ne x'_t$ exactly in the second branch, which has probability $1/4$.
\end{enumerate}
\end{lemma}

\begin{proof}
(a) Every $a_k$ lies outside $L'$, and $L'$ includes $\supp(\cD_1)=\{s_0\}\cup\{b_k:k\ge1\}$.

(b) The support of $\cD_0$ is $\{s_0\}\cup F$, so the only valid output outside $F$ is $s_0$, and it is new only if $s_0\notin S$.
Each sample misses $s_0$ with probability $1/4$, so $\Pr[s_0\notin S]=4^{-n}$.

(c) The first coordinate equals $s_0$ with probability $3/4$ and $a_k$ with probability $2^{-k}/4$; the second coordinate is the same, with $b_k$ in place of $a_k$.
In the second branch, $a_K\notin L'$ and $b_K\in L'$, so $a_K\ne b_K$.
\end{proof}

\begin{remark}[The role of the common string]\label{rem:common-string}
Without the common string $s_0$, coupled samples differ in every coordinate, and Lemma~\ref{lem:lecam} gives only $e^{-\eps n}$, which vanishes at $\eps=\infty$.
The common string makes the two samples equal with probability $(3/4)^n$, and this gives the non-private part of the bound.
The weight $3/4$ is not optimized.
With weight $1-p$ on $s_0$, the proof of Theorem~\ref{thm:id-lb} gives $\max_b\iderr\ge\frac p2(1-p\theta_\eps)^n$.
\end{remark}

\begin{remark}[Comparison with the non-private bounds]\label{rem:nonprivate-lb}
At $\eps=\infty$, Theorems~\ref{thm:id-lb} and~\ref{thm:gen-lb} give $\frac18(3/4)^n$ and $\frac13(3/4)^n$ for every $n$.
\citet[Theorem~2.2]{hp26} prove $\iderr\ge e^{-5n}$ for infinitely many $n$, assuming only that two languages have private strings; our instance also uses a common string $s_0\in L\cap L'$, which makes the two samples equal with probability $(3/4)^n$.
\citet[Theorem~3.5]{hp26} prove $\generr\ge\frac14e^{-2n}$ for infinitely many $n$ under the different condition that $L\cap L'$ is finite and $\cU\setminus(L\cup L')\ne\varnothing$, which our proof does not need.
Our proofs do not use these non-private bounds.
\end{remark}

\begin{remark}[Every intersecting pair with one common string is sharp]\label{rem:sharp-general}
Let $\cC=\{L,L'\}$ satisfy the IIDP condition with $L\cap L'=\{s_0\}$, and let $(a_k)$ and $(b_k)$ enumerate $L\setminus L'$ and $L'\setminus L$ in the order of $\cU$.
Then the hard distributions of Theorem~\ref{thm:gen-lb} have $\supp(\cD_0)=L$ and $\supp(\cD_1)=L'$.
So under each of them one language of $\cC$ is good, and the other is bad, with a witness at the level of its first string outside the support; Assumption~\ref{asm:gen} holds for both.
Hence $W=\max\{w(\cC,\cD_0),w(\cC,\cD_1)\}$ is a common public bound, and it depends only on $\cC$ and the enumeration.
The target strings in $\cU_W$ under $\cD_0$ are $s_0$ and the elements $a_k$ of level at most $W$; the $k$-th element of $L\setminus L'$ has level at least $k$, so each such $a_k$ has mass $2^{-k}/4\ge2^{-W}/4$, and $m_W\ge2^{-W}/4$.
The same holds under $\cD_1$.
With the horizon $f=2$, Theorem~\ref{thm:gen-public}(a) gives $\generr\le We^{-n2^{-W}/32}+2e^{-\eps n2^{-W}/16}+n2^{-n}$ under both distributions, and Theorem~\ref{thm:gen-lb} gives $\frac13e^{-\min\{1,\eps\}n/3}$ under one of them.
So the optimal pure-DP rate on $\{\cD_0,\cD_1\}$ is $e^{-\Theta(\min\{1,\eps\}n)}$ for every $\eps\in(0,\infty)$.
In Example~\ref{ex:sharp}, $W=4$, and the mass floor $m_4=1/8$ is larger than this general bound.
\end{remark}

\section{The non-private rules as limits of the scores}\label{app:nonprivate}

The scores of Sections~\ref{sec:id} and~\ref{sec:gen} are not ad hoc.
Each one is a signed version of a non-private rule of \citet{hp26}: the rule outputs a maximizer of the score.
The private algorithms keep the score and replace the maximizer by a mechanism.

\paragraph{Identification.}
The rule of \citet{hp26} outputs the largest $i\in[f]$ with $M_S(i)>2/f$.
\paragraph{Generation.}
For each $i\in[f]$, the rule of \citet{hp26} keeps a \emph{pointer} $r_i(S)$, the smallest $k$ with $u_k\in L_i$ and $u_k\notin S$.
It outputs a new string of a language with the largest pointer.

\begin{proposition}[Non-private rules]\label{prop:limits}
Fix a sample $S$, and assume $f\ge3$.
\begin{enumerate}[label=(\alph*)]
\item The output of the identification rule maximizes $\qid(S,\cdot)$ over $[f]$.
\item There is $\eta_0>0$ such that for every $\eta\in(0,\eta_0]$, the capped-depth score with profile $\lambda_\ell=\eta\ell$ equals $\eta\,(r_i(S)-1)$ for every $i\in[f]$. So its maximizers are the languages with the largest pointer.
\end{enumerate}
\end{proposition}

\begin{proof}
(a) Let $i_0$ be the output of the rule.
It exists, because the index $1$ passes the test: $M_S(1)=1>2/f$ when $f\ge3$.
The index $i_0$ passes the test and every later index fails it, so every term of $\qid(S,i_0)$ is nonnegative, and $\qid(S,i_0)\ge0$.
Now let $i\ne i_0$.
If $i>i_0$, then $i$ fails the test, so the first term of $\qid(S,i)$ is $M_S(i)-2/f\le0$.
If $i<i_0$, then the inner minimum contains $2/f-M_S(i_0)<0$.
In both cases $\qid(S,i)\le0\le\qid(S,i_0)$.

(b) Every string of $L_i$ with index below $r_i(S)$ lies in $S$, and $u_{r_i(S)}$ does not.
So $a_{i,\ell}(S)\ge1$ for $\ell<r_i(S)$, and $a_{i,\ell}(S)=0$ for $\ell\ge r_i(S)$.
Let $\eta_0=1/\max_{i\in[f]}r_i(S)$, and let $0<\eta\le\eta_0$.
For $\ell<r_i(S)$, the term of the score at level $\ell$ is $\min\{a_{i,\ell}(S),\eta\ell\}=\eta\ell$, because $\eta\ell<1\le a_{i,\ell}(S)$; for $\ell\ge r_i(S)$, the term is $0$.
If $r_i(S)\ge2$, then the largest term is the one at $\ell=r_i(S)-1$, so $\qgen(S,i)=\eta(r_i(S)-1)$.
If $r_i(S)=1$, then every term is $0$, and again $\qgen(S,i)=0=\eta(r_i(S)-1)$.
Because $\eta>0$, the maximizers of the score are the languages with the largest pointer.
\end{proof}

The proposition shows what privacy changes.
The non-private generation rule uses a profile with increments $\eta\to0$, so its margin is a single count.
One sample can then flip it, so no finite $\eps$ makes the rule $\eps$-DP.
Our profiles have increments of order $n$ (with a public bound) or $g(n)/w^2$ (without one), and these increments are the margins in Theorems~\ref{thm:gen-public} and~\ref{thm:gen-nopublic}.

\section{Limitations, broader impacts, and more related work}\label{app:limits}

\subsection{Limitations and open problems}

\paragraph{Computation.}
Our results are information-theoretic.
The algorithms enumerate languages and strings, and we do not bound their running time.
Recent work on computational barriers for non-private generation~\citep{abck25} suggests that efficient algorithms need more structure.
The link to the DP-SGD pipeline for language models~\citep{acg+16,ltlh2022large,pvx+23,pxm+25} is also open.

\paragraph{The unit of privacy.}
We protect single strings.
User-level privacy~\citep{lsy+20,lsa+21,ghazi+23}, where one user contributes many strings, needs group privacy or new arguments.

\paragraph{Open problems.}
\begin{enumerate}
\item \emph{Which collections need a public bound?} Theorem~\ref{thm:gen-nec} shows that a finite collection with two infinite languages of finite intersection needs a public bound for exponential pure-DP rates. A characterization is open. Some collections need no bound at all. If all languages share an infinite set $K$, then $K\subseteq\supp(\cD)$ whenever some language is good, and a random new string of $K$ is valid.
\item \emph{Identification.} The gap between $e^{-o(n)}$ and $e^{-O(n)}$ is open even without privacy.
\item \emph{Small $\eps$.} For identification, privacy costs nothing for $\eps\ge2\sqrt{r(n)/n}$ (Corollary~\ref{cor:id-rates}), and Theorem~\ref{thm:id-lb} shows a cost for $\eps\le c\,r(n)/n$. The behavior between these thresholds is open.
\item \emph{Other models.} Concurrent work~\citep{mvyz26} studies private learning in the online model of \citet{gold67} and \citet{km24}. Richer objectives, such as safety constraints~\citep{akk26} or representativeness~\citep{prr25}, are natural next steps.
\end{enumerate}

\subsection{Broader impacts}\label{app:impacts}

This work is foundational and theoretical.
Language models increasingly train on user text, which often contains personal information.
Our results show what formal privacy costs in two basic language-learning tasks.
Approximate DP with $\delta\ge n^{-O(1)}$ costs nothing in the rate for both tasks; for generation without a public bound, the exact condition on the collections of Theorem~\ref{thm:gen-nec} is $\log(1/\delta)=o(n)$.
Pure DP costs nothing for identification at every fixed $\eps$.
For generation, it costs at most a constant factor in the exponent with a public bound (a factor that is necessary on hard instances), and it can cost the exponential rate without one.
These facts can inform the choice between pure and approximate DP in applications.
The work releases no models, data, or code.
One indirect risk is that someone uses a DP guarantee to justify a deployment without further review.
DP protects individual contributions to the training data, but it does not address fairness, representativeness, or content safety.
Practitioners should still audit privacy and utility end to end and combine DP with other safeguards.

\subsection{More related work}\label{app:related}

\begin{table}[t]
\centering
\caption{The price of privacy across learning tasks.
The upper block gives the number $n$ of samples that a task needs for accuracy $\alpha$ and confidence $1-\beta$.
For the PAC and counting-query rows, it gives only how $n$ depends on the domain (and on $|Q|$), with the other parameters fixed.
The lower block gives error rates in $n$ at a fixed distribution, as in Table~\ref{tab:summary}.
Compare the change within a row, not values across rows.}
\label{tab:compare}
\footnotesize
\setlength{\tabcolsep}{3.5pt}
\renewcommand{\arraystretch}{1.3}
\begin{tabularx}{\linewidth}{@{}>{\raggedright\arraybackslash}p{0.21\linewidth}>{\centering\arraybackslash\hsize=0.85\hsize}X>{\centering\arraybackslash\hsize=1.3\hsize}X>{\centering\arraybackslash\hsize=0.85\hsize}X>{\raggedright\arraybackslash}p{0.19\linewidth}@{}}
\toprule
\textbf{Task} & \textbf{Non-private} & \textbf{Pure $\eps$-DP} & \textbf{Approx.\ $(\eps,\delta)$-DP} & \textbf{What privacy changes} \\
\midrule
\multicolumn{5}{@{}l}{\emph{Sample complexity}} \\[1pt]
Hypothesis selection, $M$ candidates$^a$ & $O\bigl(\frac{1}{\alpha^2}\log\frac M\beta\bigr)$ & $O\bigl((\frac{1}{\alpha^2}+\frac{1}{\alpha\eps})\log\frac M\beta\bigr)$ & the same, up to polylog factors$^b$ & an additive term, of lower order if $\eps\ge\alpha$ \\
PAC lines$^c$ in $\mathbb Z_p^2$, $p$ prime & no $p$-dependence & $\Theta(\log p)$ & no $p$-dependence & pure DP pays $\log p$; approximate DP removes it \\
PAC thresholds$^d$ on a finite ordered domain $X$ & no $|X|$-dependence & $\Theta(\log|X|)$ & $\tilde\Theta(\log^*|X|)$ & approximate DP shrinks the cost; for infinite $X$, no private learner \\
Counting queries$^e$ $Q$ on $\{0,1\}^d$, worst case & exact for every $n$ & $\Theta(d\log|Q|)$ & $\tilde\Theta(\sqrt d\log|Q|)$ & a factor $d$ or $\sqrt d$ over the $\log|Q|$ of population accuracy \\
$d$ one-way marginals$^f$ & exact for every $n$ & $\Theta\bigl(\frac{d}{\alpha\eps}\bigr)$ & $\Theta\bigl(\frac{\sqrt{d\log(1/\delta)}}{\alpha\eps}\bigr)$ & approximate DP shrinks $d$ to $\sqrt{d\log(1/\delta)}$ \\
Uniform generation in the limit$^g$ & $d_{\mathrm{cl}}+1$ & $d_{\mathrm{cl}}+\tilde O\bigl(\frac{k}{\eps}\log\frac1\beta\bigr)$ & not studied & an additive term of order $k/\eps$ \\
\midrule
\multicolumn{5}{@{}l}{\emph{Error rate at a fixed distribution (this work; non-private rates from \citet{hp26})}} \\[1pt]
Identification$^h$ & $e^{-r(n)}$ & $e^{-r(n)}$ & $e^{-r(n)}$ & nothing at fixed $\eps$ \\
Generation, public witness bound$^i$ & $e^{-\Theta(n)}$ & $e^{-\Theta(\min\{1,\eps\}n)}$ & $e^{-\Theta(n)}$ & pure DP: a factor $\min\{1,\eps\}$ in the exponent \\
Generation, no public bound$^j$ & $e^{-\Theta(n)}$ & $e^{-o(n)}$, not exponential & $e^{-\Theta(n)}$ iff $\log(1/\delta)=o(n)$ & pure DP: the exponential rate itself \\
\bottomrule
\end{tabularx}

\end{table}

\paragraph{Language identification.}
\citet{lzz08} survey identification in the limit.
\citet{cpt25} study list identification and show that $k$ guesses per step strictly enlarge the identifiable classes.
\citet{psv26} show that computational traces make all recursively enumerable languages identifiable.
\citet{pf25} extend Gold's framework with computational observations and restricted input sources.

\paragraph{Language generation.}
\citet{lrt24} place generation in a learning-theoretic hierarchy.
\citet{kmv25,kmv26}, \citet{cp24}, \citet{prr25}, and \citet{kw25a,kw25,kw26} study the tension between diversity and hallucination; \citet{cp25b} study Pareto-optimal trade-offs.
\citet{rr25}, \citet{bpz25}, \citet{mvyz25}, and \citet{lz26} study robustness to noise.
\citet{hkmv25} and \citet{bpz25} show that generatability is not closed under finite unions for uncountable collections, and \citet{kmsv25} study hallucination detection.
\citet{abck25} analyze computational and sample-complexity barriers.
Recent extensions include generation in metric spaces~\citep{lrt26} and safe generation~\citep{akk26}.
\citet{rvs26} study model collapse and show that replay separates weaker notions of generation.

\paragraph{Private learning and language models.}
Private PAC learning~\citep{klnrs11,bnsv15,bun20} and private estimation~\citep{cch12,aak21,bks22} have a large literature.
DP-SGD~\citep{acg+16} is widely used to train and fine-tune language models~\citep{acn+22,ynb+22,ltlh2022large}, and \citet{pvx+23} give a practical guide.
\citet{sss+23} study scaling laws for private training, and \citet{yue+23} give recipes for private synthetic text.
\citet{ghazi+23} study user-level privacy with few examples per user, and \citet{pxm+25} give a guide to private synthetic data.
\citet{gzw+25} synthesize private few-shot demonstrations for in-context learning, and \citet{lb25} remove the clipping hyperparameter of DP-SGD.
\citet{blv+26} measure the privacy of DP-adapted language models with membership-inference and extraction attacks, and \citet{rppn26} propose a token-level private inference mechanism.
These works are empirical; we give information-theoretic guarantees.

\subsection{The price of privacy in other learning tasks}\label{app:compare}

Table~\ref{tab:compare} compares the price of privacy in our two tasks with known prices in other learning tasks.
The upper block gives sample complexities, and the lower block gives error rates in $n$ at a fixed distribution.
These are different quantities, so the table compares how privacy changes each task, not values across rows.
The notes at the end of this subsection give the sources and the exact conditions for each row.

\paragraph{Identification behaves like hypothesis selection.}
Both tasks select a candidate by a score with a margin.
In hypothesis selection, the margin is the accuracy $\alpha$, and the statistical term $\log M/\alpha^2$ dominates the privacy term $\log M/(\alpha\eps)$ once $\eps\ge\alpha$.
In identification, the margin is $1/f(n)$, and the statistical exponent $n/(8f^2)$ is the smaller one once $\eps\ge1/(2f)$ (Theorem~\ref{thm:id}).
Because $f(n)\to\infty$, this condition holds at every fixed $\eps$ for large $n$, so privacy is free.
This is the first comparison of Section~\ref{sec:discussion}, seen from the side of sample complexity.

\paragraph{Generation with a public bound pays the usual factor.}
At a fixed distribution, the rate $e^{-\Theta(\min\{1,\eps\}n)}$ means that a failure probability $\beta$ needs $\Theta(\log(1/\beta)/\min\{1,\eps\})$ samples, against $\Theta(\log(1/\beta))$ without privacy.
This is the familiar factor $1/\eps$ of pure DP, as in the privacy term of hypothesis selection.
Approximate DP removes it.

\paragraph{Generation without a public bound is a rate analog of the PAC separations.}
In PAC learning, pure DP can pay a cost that depends on the domain, even when the Littlestone dimension is $2$~\citep{fx15}.
Approximate DP removes this cost for lines and shrinks it for thresholds.
For thresholds on an infinite domain, privacy even destroys learnability~\citep{almm19}.
In generation without a public bound, the cost of pure DP is the exponential rate itself (Theorem~\ref{thm:gen-nec}).
Approximate DP removes this cost exactly when $\log(1/\delta)=o(n)$ (Corollary~\ref{cor:gen-threshold}), so the threshold marks where approximate DP starts to behave like pure DP.

\paragraph{The pattern.}
In the upper block, approximate DP removes the price of pure DP for lines.
It shrinks the price for thresholds, counting queries, and marginals, and in the worst case it gains at most polylogarithmic factors for hypothesis selection.
In our tasks, approximate DP with $\log(1/\delta)=o(n)$ removes the price in every row.
Pure DP pays nothing, a constant factor in the exponent, or the exponential rate.
\citet{mvyz26} also remark, informally, that their techniques give almost exponential private rates in the realizable statistical model of \citet{kmv25}.
Remark~\ref{rem:realizable} is consistent with this: without a public bound, pure DP cannot reach an exponential rate in general, even for realizable distributions.

\paragraph{Notes to Table~\ref{tab:compare}.}
{\small
\begin{itemize}[leftmargin=1.2em,itemsep=1pt]
\item[$^a$] Output within $3\,\mathrm{OPT}+\alpha$ in total variation~\citep{bksw19,aak21}. If the candidates are pairwise far apart (a packing), pure DP needs $\Omega(\log M/\eps)$ samples~\citep{bksw19}.
\item[$^b$] In the worst case. For $\delta$ of order $1/n$, product distributions on $\{0,1\}^d$ need $\tilde\Omega(d/(\alpha\eps))$ samples~\citep{klsu19}, and they have an $\alpha$-cover with $\log M=O(d\log(d/\alpha))$. For some candidate sets, approximate DP does much better~\citep{bksw19}.
\item[$^c$] Realizable. The concepts are the indicators of the $p^2$ lines $y=ax+b$ in $\mathbb Z_p^2$; the VC and Littlestone dimensions are $2$. Pure DP: $\Omega(\log p/(\alpha\eps))$ for every learner, proper or not~\citep{fx15}, and $O((\log p+\log(1/\beta))/(\alpha\eps))$~\citep{klnrs11}. Approximate DP: $O(\log(1/\delta)\log(1/\beta)/(\alpha\eps))$ for every $p$, by an improper learner~\citep{fx15}.
\item[$^d$] Realizable. Pure DP: \citet{klnrs11} (upper) and \citet{fx15} (lower, also for improper learners). Approximate DP: \citet{bnsv15,klmns20,clnss23} (upper) and \citet{almm19} (lower, for constant $\eps$ and $\delta=O(1/(n^2\log n))$); $\tilde\Theta$ hides factors polylogarithmic in $\log^*|X|$. On an infinite domain, no approximate-DP learner exists~\citep{almm19}.
\item[$^e$] The database size for answering every query of $Q$ within $\pm\alpha$, at fixed $\alpha$ and $\eps$, and $n^{-O(1)}\le\delta=o(1/n)$. The upper bounds hold for every $Q$, and the lower bounds for some $Q$ of each size in a range of large sizes. Pure DP: $O(d\log|Q|/(\eps\alpha^3))$~\citep{blr13} and $\Omega(d\log|Q|/(\eps\alpha^2))$~\citep{hardt11}; the right power of $1/\alpha$ is open~\citep{nik23}. Approximate DP: $O(\sqrt{d\log(1/\delta)}\log|Q|/(\eps\alpha^2))$ by private multiplicative weights~\citep{hr10,dr14}, and $\tilde\Omega(\sqrt d\log|Q|/(\eps\alpha^2))$~\citep{buv14}. Without privacy, $\Theta(\log|Q|/\alpha^2)$ samples give accuracy with respect to the population~\citep{buv14}.
\item[$^f$] The database size for average error $\alpha$ with respect to the empirical marginals; for $\eps\in(0,1)$, small constant $\alpha$, and $2^{-o(n)}\le\delta\le n^{-1-\Omega(1)}$. Lower bounds: \citet{ht10} (pure) and \citet{su17} (approximate). The Laplace and Gaussian mechanisms match them.
\item[$^g$] Online model, adversarial order, realizable; $k$ languages of closure dimension $d_{\mathrm{cl}}$; pure DP under continual release. Here $n$ is the time after which the generator succeeds~\citep{lrt24,mvyz26}. The bound $d_{\mathrm{cl}}+\Omega(k/\eps)$ holds for some collections.
\item[$^h$] Agnostic, with an optimal language. For every $r(n)=o(n)$ with $r(n)\ge\log n$, at every fixed $\eps$, and $\delta\ge n^{-O(1)}$ (Corollary~\ref{cor:id-rates}).
\item[$^i$] In the generation rows, $\Theta$ combines an upper bound at every distribution with a lower bound on a hard instance (Table~\ref{tab:summary} separates them).
\item[$^j$] Finite collections with two infinite languages of finite intersection (Corollary~\ref{cor:gen-threshold}).
\end{itemize}
}

\end{document}